\documentclass[Afour,sageh,times]{sagej}

\usepackage{moreverb,url}
\usepackage[colorlinks,bookmarksopen,bookmarksnumbered,citecolor=red,urlcolor=red]{hyperref}

\newcommand\BibTeX{{\rmfamily B\kern-.05em \textsc{i\kern-.025em b}\kern-.08em
T\kern-.1667em\lower.7ex\hbox{E}\kern-.125emX}}

\usepackage{times}
\usepackage{xcolor,colortbl}
\definecolor{lavender}{rgb}{0.9, 0.9, 0.98}
\usepackage{bbding}
\usepackage{verbatim}
\usepackage{amsmath}
\usepackage{algorithm}
\usepackage[noend]{algpseudocode}
\usepackage{graphicx}
\usepackage{amssymb}
\usepackage{amsmath}
\usepackage{amsthm}
\usepackage{array}
\usepackage[draft]{fixme}
\usepackage[english]{babel}
\usepackage{verbatim}
\usepackage{tikz}
\usepackage{multirow}
\usepackage{xspace}
\usepackage{color}
\usepackage{tikz}
\usepackage{svg}
\usepackage{caption}
\usepackage{subcaption}
\usepackage[colorinlistoftodos]{todonotes}
\usepackage[font=small,labelfont=bf,tableposition=top]{caption}
\usepackage[normalem]{ulem}
\usepackage{pifont}
\newcommand{\cmark}{\ding{51}}%
\newcommand{\xmark}{\ding{55}}%
\newcommand\LyonsStrikeout{\bgroup\markoverwith
{\textcolor{cyan}{\rule[0.5ex]{2pt}{0.8pt}}}\ULon}
\usepackage{bm}
\usepackage{bbm}

\makeatletter
\def\BState{\State\hskip-\ALG@thistlm}
\makeatother                           

\DeclareCaptionLabelFormat{andtable}{#1~#2  \&  \tablename~\thetable}

\usepackage{booktabs, tabularx}
\usepackage{multirow}

\usepackage{xspace}

\newcounter{defCounter}
\newcounter{LemCounter}

\newtheorem{theorem}{Theorem}

\newtheorem{definition}[defCounter]{Definition}
\theoremstyle{definition}

\newtheorem{lemma}[LemCounter]{Lemma}

\newtheorem{problem}{Problem}

\newtheorem{remark}{Remark}

\theoremstyle{definition}
\newtheorem{assumption}{Assumption}

\begin{document}
\runninghead{Boldrer, Serra-G\'omez, Lyons, Kr\'atk\'y, Alonso-Mora, Ferranti}
\title{Rule-based Lloyd algorithm for multi-robot motion planning and control
with safety and convergence guarantees} %

\author{Manuel Boldrer\affilnum{1}, \'Alvaro Serra-G\'omez\affilnum{3}, Lorenzo
Lyons\affilnum{2}, V\'it Kr\'atk\'y\affilnum{1}, Javier Alonso-Mora\affilnum{2},
Laura Ferranti\affilnum{2}} 
\affiliation{\affilnum{1} Department of Cybernetics, Czech Technical University,
Karlovo namesti 13, 12135, Prague, Czechia \\
\affilnum{2} Department of Cognitive Robotics, Delft University of Technology,
Mekelweg 2, Delft, The Netherlands \\
\affilnum{3}  Leiden Institute of Advanced Computer Science, Leiden University,
Leiden, The Netherlands}
\corrauth{Manuel Boldrer}  
\email{boldrman@cvut.cz}



\begin{abstract} 
This paper presents a distributed rule-based Lloyd algorithm (RBL) for
  multi-robot motion planning and control. The main limitations of the basic
  Loyd-based algorithm (LB) concern deadlock issues and the failure to address
  dynamic constraints effectively. Our contribution is twofold. First, we show
  how RBL is able to provide safety and convergence to the goal region without
  relying on communication between robots, nor
  synchronization between them. We considered different dynamic
  constraints with control inputs saturation. Second, we show that the
  Lloyd-based algorithm (without rules) can be successfully used as a safety
  layer for learning-based approaches, leading to non-negligible benefits. We
  further prove the soundness, reliability, and scalability of RBL through
  extensive simulations, comparisons with the state of the art, and
  experimental validations on small-scale car-like robots, unicycle-like
  robots, omnidirectional robots, and aerial robots on the field.
\end{abstract}
\maketitle
\section{Introduction} \label{sec:Introduction} The rise of mobile robotics is
poised to have a significant impact on our daily lives in the near future. One
essential skill that autonomous robots must have is the ability to move from
one location to another. However, this task can be quite challenging,
particularly in environments shared with other robots. An effective motion
planning and control algorithm should prioritize three key aspects: safety,
security, and convergence towards the intended destination.

Safety is of utmost importance, each robot should be equipped with collision
avoidance capabilities to prevent accidents with other robots or assets.
Additionally, security measures, such as reducing (or completely removing)
reliance on the communication network or employing algorithms that are resilient
against attacks and packet losses/delays, should be implemented to ensure the
robustness of the system. Finally, successful convergence towards the intended
destination requires avoiding the occurrence of deadlocks and live-locks, where
deadlocks refer to situations where robots are stuck and unable to proceed due
to conflicting actions, while live-locks involve continuous repetitive actions
that prevent progress towards the desired goal location. In this work we address
all the aforementioned issues by employing a novel Rule-Based Lloyd (RBL)
solution. \subsection{Related works} Multi-robot motion planning and control can
be broadly categorized into two main approaches: \textit{centralized} and
\textit{distributed} methods. Centralized approaches in multi-robot motion
planning and control involve a central unit responsible for computing control
inputs for all the robots in the system. These methods offer the advantage of
achieving optimal solutions more easily. However, they face scalability issues
when the number of robots increases, and they have to rely on a dependable
communication network to exchange real-time information between the coordinator
and the robots. Centralized approaches can be particularly valuable in
well-structured environments such as
warehouses~\cite{sharon2015conflict,li2021lifelong}, where optimizing the
solution is of utmost importance.

In contrast, distributed approaches in multi-robot motion planning and control
involve individual robots computing their own control inputs using local
information. While these methods may yield sub-optimal solutions, they offer
advantages in terms of scalability and robustness compared to centralized
approaches. Distributed methods are capable of handling a larger number of
robots and are more resilient to communication failures or limitations. They
allow robots to make decisions based on local perception, which can be
beneficial in dynamic and uncertain environments.

After opting to explore distributed methods due to their significant advantages
over centralized approaches, we can identify three primary categories for
multi-robot motion planning and control~\cite{cheng2018autonomous}:
\textit{reactive}, \textit{predictive planning}, \textit{learned-based}
methods. 

\subsubsection{Reactive methods} are  myopic by definition, hence each action
is taken on the basis of the current local state. Despite being computationally
efficient and usually easy to implement, deadlock conditions (local minima) as
well as live-lock may occur, preventing the robots to reach the goal positions.
The most popular approaches rely on velocity
obstacles~\cite{van2008reciprocal,van2011reciprocal,alonso2013optimal,boldrer2020socially1,arul2021v,godoy2020c,giese2014reciprocally,zheng2024mcca},
force fields~\cite{helbing1995social, boldrer2020socially, gayle2009multi},
barrier certificates~\cite{wang2017safety,celi2019deconfliction, zhou2021ego},
and dynamic window~\cite{fox1997dynamic}.

\subsubsection{Predictive planning methods} use information about the
neighboring robots in order to plan optimal solutions. Even though they can
substantially increase the performance metrics and may avoid deadlocks by
increasing the prediction horizon, usually they can be computationally
expensive and/or the robots need to exchange a large amount of information
through the communication network, which may not be highly reliable. Moreover,
these methods may fall in live-lock conditions, which prevent the robots to
converge to the goal positions. Solutions based on assigning priorities
(sequential solutions) or synchronous re-planning are adopted to avoid this
problem, but in the former case it introduces a hierarchy among the agents, in
the latter a global clock synchronization is required. The most popular
approaches rely on model predictive control
(MPC)~\cite{ferranti2022distributed,tajbakhsh2023conflict,zhu2019chance,arul2023ds,luis2019trajectory,kloock2023coordinated,10050739,trevisan2024biased},
elastic bands~\cite{chung2022distributed}, buffered Voronoi
cells~\cite{zhou2017fast,zhu2019b,abdullhak2021deadlock,pierson2020weighted},
legible motion~\cite{smithICRA2023,mavrogiannis2019multi} and linear spatial
separations~\cite{csenbacslar2023rlss}. Other interesting and recent solutions
are provided
by~\cite{park2022decentralized,9976221,vcap2015complete,kondo2023robust,csenbacslar2023dream,chen2022recursive,chen2023asynchronous,adajania2023amswarm,toumieh2024high},
all of them relying on a communication network. \subsubsection{Learning-based
methods} are promising and many researchers are investigating their
potential~\cite{orr2023multi,chen2017decentralized,chen2017socially,han2022reinforcement,xie2023drl,han2020cooperative,fan2020distributed,tan2020deepmnavigate,long2018towards,brito2021go,chen2019crowd,everett2018motion,qin2023srl,li2020graph,chen2023toward}.
    However, the weakness of these approaches is related to the lack of
    guarantees, especially for safety and convergence. In addition, the
    solution that these approaches provide are often non-explainable and
    generalization to situations different from the training scenario are hard
    to provide (e.g., number of robots, robots' encumbrance, robots' dynamics
    and velocities). With respect to the state of the art algorithms, our
    solution allows the robots to avoid collisions and converge to their goal
    regions without relying on communication, without relying on the neighbors'
    control inputs, and without the need of synchronization. In the literature,
    the main approaches that rely only on neighbors positions and encumbrances,
    and hence do not require communication between the robots, are
    BVC~\cite{zhou2017fast}, RLSS~\cite{csenbacslar2023rlss} and
    SBC~\cite{wang2017safety}, but none of them achieves a success rate of
    $1.00$ in very crowded scenarios. In fact, we tested BVC and SBC in
    environments with a crowdness factor $\eta= 0.452$; both of them result in
    a success rate SR = 0.00. On the other hand, RLSS has a higher success
    rate; however, it can experience deadlock/live-lock issues, as well as
    safety issues in an asynchronous setting (see~\cite{csenbacslar2023dream}).
    This is a severe problem also due to their required computational time,
    which in average is greater than $100$~(ms).

 Our algorithm prove to be a valid solution in scenarios where communication
 between robots is unfeasible. This situation can arise from several reasons,
 such as security concerns, cost reduction initiatives, the necessity for
 interactions between diverse robots from different companies, adversarial
 environmental conditions, or simply for technological
 limitations~\cite{gielis2022critical}. A safe and networkless motion planning
 and control algorithm can find plenty of applications including but not
 limited to agriculture-related applications, environmental monitoring,
 delivery services, warehouse and logistics, space exploration,  search and
 rescue among others.

\subsection{Paper contribution and organization}
This paper proposes a distributed algorithm for multi-robot motion planning and
control. Our contribution is twofold.

Firstly, we synthesize what we called the rule-based Lloyd algorithm (RBL), a
solution that can guarantee safety and convergence towards the goal region.
With respect to~\cite{boldrer2020lloyd}, here we focused on
the interaction between robots and: \emph{1)} We drastically improve the
multi-robot coordination performance, that is, we increased the success rate
from 0.00 to 1.00 in scenarios with more than $20$ robots; \emph{2)} We
provide, not only safety, but also sufficient conditions for convergence
towards the goal regions; and \emph{3)} We extend the algorithm to account for
dynamic constraints  and control inputs saturation by leveraging on model
predictive controller (MPC). Secondly, we show how the basic Lloyd-based
algorithm (without rules) can be used effectively as a \textit{safety layer} in
learning-based methods. We present the Learning Lloyd-based algorithm (LLB),
where we used reinforcement learning~\cite{serragomez2023a} to synthesize a
motion policy that is intrinsically safe both during learning and during
testing. We show that it benefits the learning phase, enhancing performance,
since safety issues are already solved. By using the Lloyd-based support, in
comparison to the same learning approach without the support (pure learning),
we measure an increase of the success rate from 0.56 to 1.00 in simple
scenarios with $5$ robots, until an increase from 0.00 to 1.00 in more
challenging scenarios with $50$ robots. Notice that, while RBL has guarantees
with respect to the convergence to the goal, on the other hand, LLB does not
ensure convergence. Nevertheless, it may perform better in terms of travel time
in certain scenarios. 

The paper is organized as follows. Sec. II provides the problem description.
Sec. III describes the proposed algorithms. Sec. IV  provides simulation
results and an updated comparison with the state of the art. Sec. V provides
experimental results. Finally, Sec. VI concludes the paper.

\section{Problem description}
\label{sec:problem description}

Our goal is to solve the following problem:
\begin{problem}
    Given $N$ robots, we want to steer each of them from the initial position
$p_i(0)$ towards a goal region defined as the ball set
  $\mathcal{B}(e_i,\varepsilon)$ centered in $e_i=[e_i^x,e^y_i]^\top$, with
  radius $\varepsilon$. We want to do it reliably in a safe and
  communication-less fashion.  
\end{problem}
\begin{assumption}\label{ass:assumption1}
  Each robot knows its position $p_i$, its encumbrance $\delta_i$, which
  defines the radius of the circle enclosing the robot, as well as the positions
  and encumbrances of its neighboring robots, i.e., $j \in \mathcal{N}_i$ if
  $\|p_i-p_j\|\leq 2 r_{s,i}$, where $\mathcal{N}_i$ denotes the set comprising
  the neighbors of the $i$-th robot, while $r_{s,i} = r_s$ is assumed to be the
  same for all the robots and it is defined as half of the sensing radius. In
  practice, both localization and neighbor detection, can be done by means of a
  localization module (e.g., GPS, encorders, cameras, IMU) and a vision module
(e.g., cameras, Lidar).  \end{assumption} \begin{assumption} To provide
  convergence guarantees, we assume that the mission space is an unbounded
  convex space. Notice that the algorithm would preserve safety even
  in cluttered environments~\cite{boldrer2022multi}, however it would entail
  making extra assumptions to prove the convergence.
\end{assumption}

\section{Approach}
\label{sec:approach}
The proposed approach is based on the Lloyd algorithm~\cite{lloyd1982least}. 
The main idea is to consider the following cost function
\begin{equation}\label{eq:jcov}
    J_{\operatorname{cov}}({p})=\sum_{i=1}^n \int_{\mathcal{V}_i}\left\|q-p_i\right\|^2 \varphi_i(q) d q,
\end{equation}
which was originally used in~\cite{cortes2004coverage} for static coverage
control, where the positions of the robots are represented by $p= [p_1,\dots,
p_N]^\top$, $p_i = [x_i,y_i]^\top$, the mission space is denoted by
$\mathcal{Q}$, the weighting function that measures the importance of the
points $q \in \mathcal{Q}$ is denoted by $\varphi_i(q) : \mathcal{Q}
\rightarrow \mathbb{R}_+$ and $\mathcal{V}_i$ indicates the Voronoi cell for
the $i$-th robot, which is defined as \begin{equation}\label{eq:voro}
    \mathcal{V}_i(p)=\left\{q \in \mathcal{Q} \mid\left\|q-p_i\right\| \leq\left\|q-p_j\right\|, \forall j \neq i\right\}.
\end{equation}
Under the assumption of single integrator dynamics $\dot{p}_i=u_i$, and by
following the gradient descent $\nabla J_{\text{cov}} =\frac{\partial
J_{\operatorname{cov}}(p)}{\partial p_i}$, it follows the proportional control
law
\begin{equation}\label{eq:lb}
    \dot{p}_i\left(\mathcal{V}_i\right)=-k_{p,i}\left(p_i-c_{\mathcal{V}_i}\right),
\end{equation}
where the tuning parameter $k_{p,i}>0$ is a positive value, and 
\begin{equation}\label{eq:centroiddef}
c_{\mathcal{V}_i}=\frac{\int_{\mathcal{V}_i} q \varphi_i(q) d q}{\int_{\mathcal{V}_i} \varphi_i(q) d q}
\end{equation}
is defined as the centroid position computed over the $i$-th Voronoi cell. It
can be proved that by imposing~\eqref{eq:lb} and assuming a time-invariant
$\varphi_i(q)$, each robot converges asymptotically to its Voronoi centroid
position. Our claim is that, by modifying the geometry of the cell
$\mathcal{V}_i$ and by properly designing a weighting function $\varphi_i(q)$,
we can obtain an efficient and effective algorithm for multi-robot motion
planning and control. Notice that, while the cell geometry to provide safety
was introduced by the basic LB
algorithm~\cite{boldrer2020lloyd,boldrer2019coverage,boldrer2021graph,
boldrer2022multi,boldrer2022unified},
one of the novelty introduced in this work, besides the use of the algorithm as
a safety layer for learning-based algorithm and the introduction of an MPC
formulation to deal with dynamical constraints, is the addition of rules in the
shaping of the weighting function $\varphi_i(q)$, which are crucial in order to
increase the performance and provide convergence guarantees to the goal
regions.

\subsection{Cell geometry $\mathcal{A}_i$} The reshape of the cell geometry is
necessary and sufficient to provide safety guarantees. By relying on the
classical Voronoi partition~\eqref{eq:voro}, we do not account for the
encumbrances of the robots in the scene, and neither for the limited sensing
range of the robots' sensors. Hence, as was proposed in
~\cite{boldrer2022unified, boldrer2020lloyd}, let us define the cell
$\mathcal{A}_i = \{\tilde{\mathcal{V}}_i \cap \mathcal{S}_i\}$, where

\begin{equation}
    \tilde{\mathcal{V}}_i=\left\{\begin{array}{l} \left\{q \in \mathcal{Q}
      \mid\left\|q-p_i\right\| \leq\left\|q-p_j\right\|\right\} \text { if }
      \frac{\left\|p_i-p_j\right\|}{2} \geq \Delta_{i j}   \\ \left\{q \in
      \mathcal{Q} \mid\left\|q-p_i\right\|
      \leq\left\|q-\tilde{p}_j\right\|\right\} \text { otherwise, }
    \end{array}\right.\
\end{equation}

$\forall  j \in \mathcal{N}_i$, where $\Delta_{i j}=\delta_j+\delta_i$
is the sum of the radius of encumbrance of robots $i$ and $j$,   and
$\tilde{p}_j=p_j+2\left(\Delta_{i j}-\frac{\left\|p_i-p_j\right\|}{2}\right)
\frac{p_i-p_j}{\left\|p_i-p_j\right\|}$, takes into account the encumbrance of
the robots, while

\begin{equation}
    \mathcal{S}_i=\left\{q  \in \mathcal{Q} \mid\left\|q-p_i\right\| \leq r_{s, i}\right\}.
    \end{equation}
\begin{theorem}[Safety]
  By imposing the control $\dot{p}_i(\mathcal{A}_i)$ in~\eqref{eq:lb},
  independently from the shape of the weighting function $\varphi_i(q)$, 
  collision avoidance is guaranteed at every instant of time.
\end{theorem}
    
\begin{proof} Since the set $\mathcal{A}_i = \{\tilde{\mathcal{V}}_i \cap
  \mathcal{S}_i\}$ is a convex set, then the centroid position
  $c_{\mathcal{A}_i} \in \mathcal{A}_i$. By definition of convex set, there
  exists a straight path from $p_i$ towards the centroid $c_{\mathcal{A}_i}$.
  Since $p_i, c_{\mathcal{A}_i} \in \{\tilde{\mathcal{V}}_i \cap
  \mathcal{S}_i\}$, by exploiting again the definition of convex set,  we ensure that
  the path towards the centroid is a safe path (i.e., no other robots on the
  path towards the centroid), hence the control law  $\dot{p}_i(\mathcal{A}_i)$
  as~\eqref{eq:lb} always preserves safety. \end{proof}

\subsection{The weighting function $\varphi_i(q)$}
While the safety can be guaranteed by properly shaping the cell geometry, the
convergence and the performance mainly depend on the definition of the function
$\varphi_i(q)$. We adopt a Laplacian distribution centered at $\bar{p}_i$, which
represents a point in the mission space that ultimately needs to converge to the
goal location. Then we define the function $\varphi_i(q)$ as follows:
\begin{equation}
\label{eq:phi}
   \varphi_i(q) = 
        \text{exp}{ \left( -\frac{\|q-\bar{p}_i\|}{\beta_i}\right)} 
\end{equation}
where
\begin{equation}
\label{eq:rho} \dot{\beta}_i(\mathcal{A}_i) = \begin{cases} -\beta_i \,\,\,
&\text{if}\,\, \|c_{\mathcal{A}_i}-p_i\| <d_{1} \, \land \,\\
&\|c_{\mathcal{A}_i} -c_{\mathcal{S}_i}\|> d_{2} \\ -(\beta_i-\beta^D_i)
&\text{otherwise.} \end{cases}
\end{equation}
\begin{equation} \label{eq:wp} \begin{split} \dot{\bar{p}}_i &= \begin{cases}
  -(\bar{p}_i -R^{p_i}(\pi/2-\epsilon)e_i) \,\,\, &\text{if}\,\,
  \|c_{\mathcal{A}_i}-p_i\| < d_{3} \, \land \,\\ &\| c_{\mathcal{A}_i} -
  c_{\mathcal{S}_i}\|> d_{4} \\ -(\bar{p}_i-e_i) &\text{otherwise,} \end{cases}
  \\ \bar{p}_i &= e_i  \,\,\, \text{if} \,\,\, \| p_i -
  \bar{c}_{\mathcal{A}_i}\| > \| p_i - c_{\mathcal{A}_i}\| \, \land \,
  \bar{p}_i = R^{p_i}(\pi/2-\epsilon)e_i. \\ \end{split} \end{equation}
Notice that $\beta_i^D, d_1, d_2, d_3, d_4$ are positive scalar values,
$R^{p_i}(\theta)$ indicates the rotation matrix centered in $p_i$, $\epsilon$ is a small number, $e_i$
is the $i$-th robot final goal position, while $\bar{{c}}_{\mathcal{A}_i}$ is
the centroid position computed over the cell $\mathcal{A}_i$ with $\bar{p}_i
\equiv e_i$.
For the sake of clarity in Figure~\ref{fig:vorocells} we depicted the
main variables and parameters for the $i$--th robot.
\begin{figure}
  \centering
  \includegraphics[width=1.0\columnwidth]{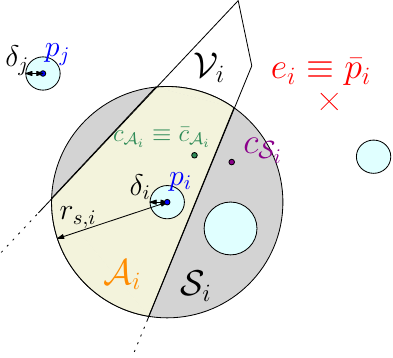}
  \caption{We depicted the main parameters and the variables associated with
  robot $i$. The robots are indicated with blue circles. Notice that the
  centroids' position strongly depend on the spreading factor $\beta_i$
  in~\eqref{eq:rho}. In particular, $\beta_i \rightarrow \infty$ cancels any
  attraction towards the goal $\bar{p}_i$ (e.g., $c_{\mathcal{S}_i} \equiv
  p_i)$. On the other hand, $\beta_i \rightarrow 0$ pushes the centroids towards
  the goal position $\bar p_i$, keeping them inside their cells.}
  \label{fig:vorocells}
\end{figure}
In the following we report some Definitions and Lemmas necessary to 
prove convergence in Theorem~\ref{pr:conv}.

\begin{definition}[Deadlock] \label{de:deadlock} The $i$--th robot is in a
  deadlock condition, if there exists a time $t_0$ such that $\forall t\geq
  t_0$, $p_i(t_0) = p_i(t)$
  and $p_i(t) \notin \mathcal{B}(e_i,r_{s,i})$. This condition is verified when
  $p_i \equiv c_{\mathcal{A}_i} \equiv \bar{c}_{\mathcal{A}_i}$ and $p_i \notin
  \mathcal{B}(e_i,r_{s,i})$, $\forall t \geq t_0$. In plain words, when the
  robot is outside of its goal region and its control inputs are zero for the
  rest of the mission. \end{definition}
\begin{definition}[Live-lock]
The $i$--th robot is in a live-lock condition if it exists a time $t_0$, such
that $\forall t>t_0$, the control input sequence $\{u_i(t_k)\}_{k=0}^{n}$ is
repeated and $p_i(t) \notin  \mathcal{B}(e_i,r_{s,i})$.
\end{definition}
\begin{lemma}[Deadlock in Symmetries]\label{lem:symmetry} By imposing the
  control law $\dot{p}_i(\mathcal{A}_i)$ in~\eqref{eq:lb}, and a time-invariant
  weighting function~\eqref{eq:phi} with a finite and positive value of $\beta_i
  =\beta^D>0$ and $\bar{p}_i= e_i$, if $\|p_i - e_i\|> r_{s,i}, \,\forall i
  =1,\dots,N$, according to Definition~\ref{de:deadlock}, deadlock can occur only
  if the set $\mathcal{A}_i$ is symmetric in the directions orthogonal to the goal
position. \end{lemma}
\begin{proof} It can be recognized from the centroid
  definition~\eqref{eq:centroiddef}. If $\|p_i - e_i\|> r_{s,i}$, and the
  $i$--robot does not interact with others, the weighting function is
  asymmetrically distributed over the $\mathcal{A}_i$ set (because of the
  attraction towards the goal $\bar{p}_i$). However, at the same time, it
  presents a symmetric distribution in the directions orthogonal to the goal
  direction $\frac{\bar{p}_i-p_i}{\|\bar{p}_i-p_i\|}$, according
  to~\eqref{eq:phi}. Let us assume the $i$--th robot to be in a scenario where
  other robots modify the shape of $\mathcal{A}_i$. If the resulting
  $\mathcal{A}_i$ cell is not symmetric in the directions orthogonal to the goal
  direction $\frac{\bar{p}_i-p_i}{\|\bar{p}_i-p_i\|}$, then, as a consequence of
  definition in~\eqref{eq:centroiddef}, $p_i \neq c_{\mathcal{A}_i}$. Hence,
  deadlock conditions can be verified only for $\mathcal{A}_i$ sets that are
  symmetric in the direction orthogonal to $\bar{p}_i$. Notice that this is true
  given that $\beta_i>0$, in fact in the case where $\beta_i=0$, other
  theoretical deadlock configuration may occur, as we will show.
 \end{proof} 

\begin{theorem}[Convergence]\label{pr:conv} 
  Let us assume to have an unbounded convex space and that the initial and goal
  positions are adequately apart from each other, i.e.,
  $\|p_i(0)-p_j(0)\|>\Delta_{ij}$,  $\|e_i-e_j\|> \Delta_{ij}$, $\forall i,j
  =1,\dots ,N$ with $i \neq j$.
  
  By imposing the control law $\dot{p}_i(\mathcal{A}_i)$ in~\eqref{eq:lb}
  and~\eqref{eq:phi},~\eqref{eq:rho},~\eqref{eq:wp} to compute the centroid
  position in~\eqref{eq:centroiddef}, and by selecting $d_1=d_3$, $d_2=d_4$
  where $0<d_1<\|p_i - c_{\mathcal{S}_i}\|$
  and $0<d_2 < \|p_i - c_{\mathcal{S}_i}\|$, with $\beta^D_i>0$, we have that
  $\lim_{t \rightarrow \infty} \|p_i(t) -e_i\| < r_{s,i} , \forall i = 1,\dots ,
  N.$
\end{theorem}

\begin{proof}
  The proof is structured as follows: first we prove deadlock avoidance,
  then live-lock avoidance guarantees and finally convergence to the goal
  positions.

   By accounting for the effects of~~\eqref{eq:rho}, if a robot is in a
  deadlock condition (see Definition~\ref{de:deadlock}), it implies that
  $\beta_i \rightarrow 0$, and that the centroid position $c_{\mathcal{A}_i}$
  does not move. Hence, according to Definition~\ref{de:deadlock} deadlock
  occurs in two situations: $i.$ if $ p_i \equiv \bar{c}_{\mathcal{A}_i} \equiv
  {c}_{\mathcal{A}_i} \equiv \mathcal{A}_i $ and $ii.$ if $p_i \equiv
  \bar{c}_{\mathcal{A}_i} \equiv {c}_{\mathcal{A}_i} \in \partial
  \mathcal{A}_i$, where $\partial \mathcal{A}_i$ indicates the contour of the
  $\mathcal{A}_i$ set. Notice that $i. \implies ii.$ but not the contrary. The
  case $i.$ describes the condition where a robot is encircled by neighbors that
  are not moving. This is a transitory condition if the external robots are not
  in condition $ii.$, hence we can focus only on the second condition. Notice
  that the existence of external robots is ensured by the assumption on the
  environment, i.e., unbounded convex space.

  \begin{figure}
    \centering
    \includegraphics[width=1\columnwidth]{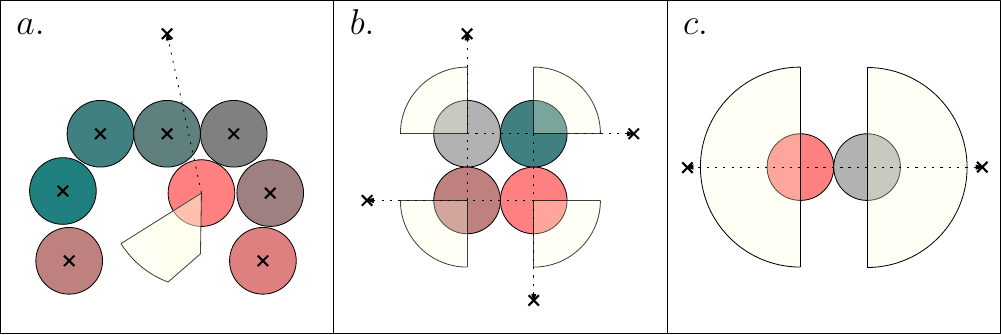}
    \caption{Theoretical deadlocking configurations, where $\beta_i = 0$,
    $\forall i=1 \dots N$. The crosses indicates the final goal positions for
    the robots (blue circles). The set $\mathcal{A}_i$ is depicted in yellow.}
    \label{fig:deadlock}
  \end{figure}
  \begin{figure}[t]
    \setlength{\tabcolsep}{0.05em}
    \centering
    \begin{tabular}{ccc}
      \includegraphics[width=.33\columnwidth]{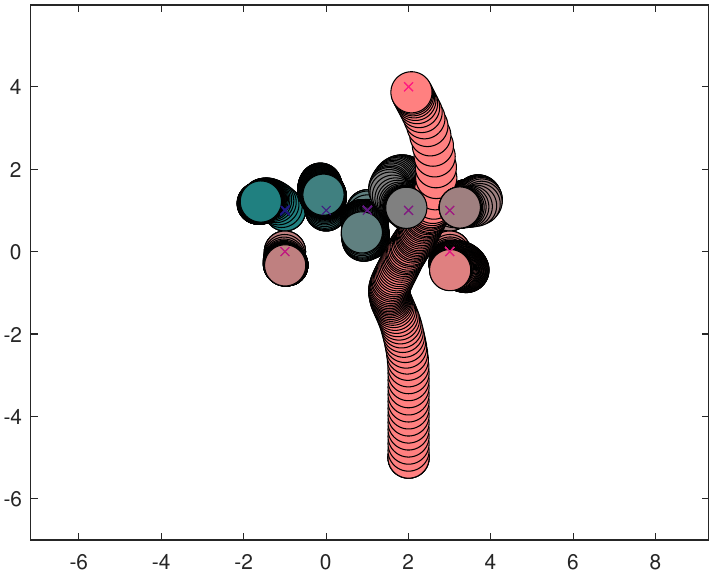} &
      \includegraphics[width=.33\columnwidth]{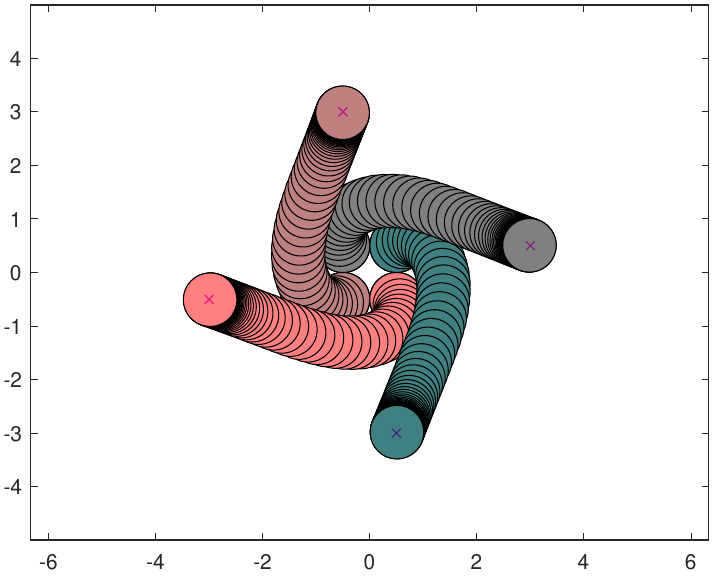}  & 
      \includegraphics[width=.337\columnwidth]{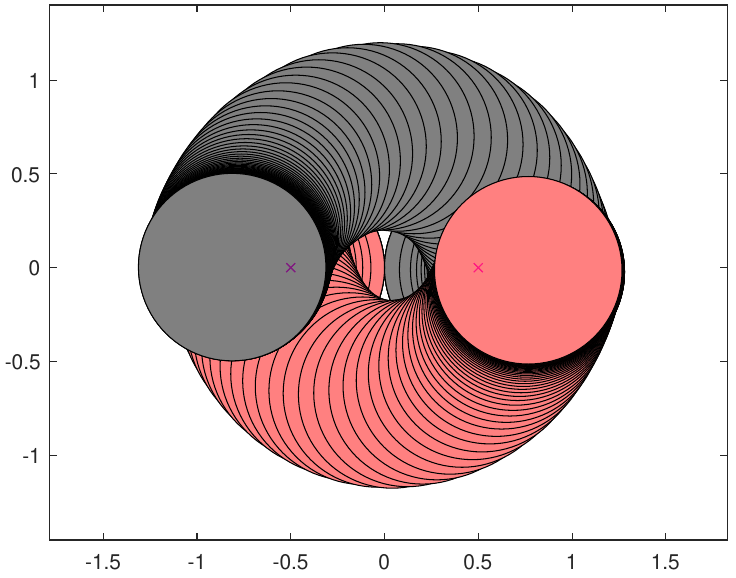} \\ 
      \footnotesize(a) & \footnotesize(b) & \footnotesize(c)
    \end{tabular}
    \caption{Simulation results starting from the theoretical deadlock
    conditions in Figure~\ref{fig:deadlock}. The desired spreading factor is
    selected as $\beta^D=0.5$ for every robot. The colored crosses
    indicates the goal positions, we depict the final robots'
    configuration.} \label{fig:deadlock1}
\end{figure}

 For the condition $ii.$ there exist three basic scenarios that satisfy
  the deadlocking conditions $p_i \equiv \bar{c}_{\mathcal{A}_i} \equiv
  {c}_{\mathcal{A}_i} \in \partial \mathcal{A}_i, \forall t \geq t_0$, with $p_i
  \notin \mathcal{B}(e_i,r_{s,i})$.  As we have shown in
  Lemma~\ref{lem:symmetry} in the case where the robots did not reach yet their
  goal positions, by assuming $\beta_i>0$, deadlock can occur only in symmetric
  conditions. On the other hand, by assuming $\beta_i = 0$, another theoretical
  configuration of deadlock is met, where robots and goal positions are located
  symmetrically. The remaining case is when only robot $i$ did not reach its
  goal position, if the robots physically hinder the $i$--th robot motion,
  deadlock conditions are met. More complex cases can be decomposed in one of
  the three scenarios or in a combination of them. We depict these scenarios in
  Fig.~\ref{fig:deadlock} for the sake of clarity, while in
  Fig.~\ref{fig:deadlock1} we report the simulation results obtained by imposing
  the robots' initial conditions to be equal to the theoretical deadlock
  conditions of Fig.~\ref{fig:deadlock}. In accordance with the simulation
  results, the deadlocks in cases \emph{a.}, \emph{b.}, and \emph{c.} do not
  occur by applying rules~\eqref{eq:rho} and~\eqref{eq:wp}.


  \paragraph*{Case a (Fig.~\ref{fig:deadlock}-a)} A certain number of robots
  have already reached their goal positions and they hinder the advancement of
  other robots that have not reached their goal locations yet. In this case, the
  value of $d_2$ in~\eqref{eq:rho} has to be selected in such a way
  that $\beta_i$ starts to decrease after a sufficient deflection from the goal,
  i.e., $ d_2\geq \alpha \delta_{\max}$, where $\delta_{\max}$ is the
  encumbrance of the biggest robot in the interaction, and $\alpha \in
  \mathbb{R}$ regulates the distance from the goal after which $\beta_i$ value
  is allowed to decrease. In practice, it is selected as any number in the
  interval $2 <\alpha< r_{s,i}/\delta_{\max}$. The lower bound is a sufficient
  condition derived geometrically, while the upper bound is set in order to
  allow the $i$--th robot to stay inside the set $\mathcal{B}(e_i,r_{s,i})$. By
  doing that, assuming $\beta_i^D >0$, each robot can be deflected from its goal
  position up to a distance $d_2$.

  This fact can be shown by computing the distance from the centroid position
  analytically. To simplify the notation, without loss of generality, we change
  reference frame, by considering $p_i = e_i =[0,0]^\top$, and our cell
  (semicircle $\mathcal{A}_i$) lying on the x-axis as our new origin. Hence, by
  recognizing that $\varphi(q)$ is symmetric with respect to the y-axis, we have
  $c_{\mathcal{A}_i}= [0, c_{\mathcal{A}_i}^y]^\top$, where
  \begin{equation}\label{eq:genericC}
    c_{\mathcal{A}_i}^y =\frac{\int_0^{\pi}\int_{0}^{r_{s, i}} z \sin \theta \,
    \text{exp}\left({-\frac{z}{\beta_i}}\right)z \,
    d\theta \, dz}{\int_{0}^{\pi} \int_{0}^{r_{s,
    i}} \text{exp}\left({-\frac{z}{\beta_i}}\right) z \, d\theta \,dz}. 
  \end{equation}
  Hence, we can write
  $$
  \left\|p_i-c_{\mathcal{A}_i}\right\|=\left\|\frac{2\int_{0}^{r_{s, i}} z^2 \,
  \text{exp}\left({-\frac{z}{\beta_i}}\right) d z}{\pi \int_{0}^{r_{s,
  i}} z \, \text{exp}\left({-\frac{z}{\beta_i}}\right) d z}\right\|.  $$
  Notice that we considered the case where we have the $i$--th robot in its goal
  position $e_i$, and a neighbor robot $j$ at distance $\Delta_{ij}
  =\delta_i+\delta_j$. By solving the integral we obtain
  \begin{small} 
  $$
  \left\|p_i-c_{\mathcal{A}_i}\right\|=\frac{2}{\pi}\left\|\frac{2 \beta_i^3 -
  \beta_i\text{exp}(-\frac{r_{s,i}}{\beta_i})(r_{s,i}^2 + 2 \beta^2 +2\beta_i
  r_{s,i})}{\beta_i^2 -
  \beta_i\text{exp}(-\frac{r_{s,i}}{\beta_i})(\beta_i+r_{s,i})}\right\|,
  $$
  \end{small}
  hence, as the value of $\beta_i$ increases, the distance
  $\|p_i-c_{\mathcal{A}_i}\|$ increases as well, which denotes the fact that each
  robot can be pushed from its goal position. The integral solution can be
  expressed in an analytic form for the special case where the goal position
  corresponds to the robot position, otherwise it needs to be computed
  numerically solving the following expression:
  \begin{equation}\label{eq:dist}\small \|p_i -
    c_{\mathcal{A}_i}\|=\left\|\frac{\int_{0}^{2 \pi}\int_{0}^{r(\theta)} z \sin
    \theta \,\text{exp}\left(\Gamma\right)z \,d z \, d \theta}{\int_{0}^{2 \pi}
    \int_{0}^{r(\theta)} \text{exp}\left(\Gamma\right)z \, dz \, d \theta}\right\|,
  \end{equation} where $\Gamma = -\frac{\sqrt{(z\cos\theta - e_i^x)^2+(z\sin\theta
  - e^y_i)^2}}{\beta_i}$ and $r(\theta)$ is defined according to the cell shape.
  Since $\beta_i>0$ implies $c_{\mathcal{A}_i} \in \text{int}(\mathcal{A}_i)$,
  then if a neighbor has $\beta_j \rightarrow 0$, pushing the $i$--th robot
  away from its goal by a distance $d_2$ is always possible.

 \paragraph*{Case b (Fig.~\ref{fig:deadlock}-b)} It is an unstable and
 unreachable equilibrium, since a small perturbation will lead leaving this
 configuration. We can argue that it is an equilibrium only if $\beta_i = 0$ for
 each robot, i.e., the centroids are positioned on the vertex of the cells; this
 is not possible since $\beta^D_i>0$ and $\beta_i$ evolves as~\eqref{eq:rho},
 $\forall i$. Notice that we are in a different condition than the one in
 Lemma~\ref{lem:symmetry}, because of the assumption on $\beta_i$.

\paragraph*{Case c (Fig.~\ref{fig:deadlock}-c)} According to
Lemma~\ref{lem:symmetry}, this case represents a deadlock in case that the
weighting function $\varphi(q)$ is time-invariant. However, the right hand side
behavior in~\eqref{eq:wp} introduces a geometric asymmetry of the
$\mathcal{A}_i$ set in the direction orthogonal to the position $\bar{p}_i$,
hence according to Lemma~\ref{lem:symmetry}, the robots avoid deadlock also in
this case. The same consideration can be done in similar conditions with $N>2$.
Notice that the values $d_1,d_2,d_3,d_4$ need to be selected properly, on the
basis of the values of $r_{s,i}$, $\beta^D$ and ${\delta_{\max}}$, i.e., it is
necessary that the conditions $\|c_{\mathcal{A}_i}-c_{\mathcal{S}_i}\| > d_2$,
$\|c_{\mathcal{A}_i}-c_{\mathcal{S}_i}\| > d_4$, $\|c_{\mathcal{A}_i}-p_i \| <
d_1$ and $\|c_{\mathcal{A}_i}-p_i \| < d_3$ can be verified, and in particular,
they can be verified at the same time. By satisfying these requirements we have
that both the rules~\eqref{eq:rho} and \eqref{eq:wp} are able to transition
between the two different dynamics during the mission, according to their
definition. Let us select $d_1=d_3$ and $d_2=d_4$. We can satisfy the
requirements by imposing $0<d_1 <d_{p_i,c_{\mathcal{S}_i}} $ and $0<d_2 <
d_{p_i,c_{\mathcal{S}_i}} $, where $d_{p_i,c_{\mathcal{S}_i}}$ is defined as
in~\eqref{eq:dist} and it can be computed numerically by selecting $\beta_i =
\beta^D_i$, $r_{s,i}>0$, and the goal position $e_i$ at a distance $r_{s,i}$
from the robot position $p_i$, hence our requirements are satisfied.

 Once we verified deadlock avoidance, we have to check for live-lock avoidance.
 Notice that, if the functions $\varphi_i(q)$ are time-invariant, cyclic motions
 are not possible, since the robots converge to their centroids
 positions~\cite{cortes2004coverage}.  In fact, by considering
 $J_{\text{cov}}(p,\mathcal{A})$~\eqref{eq:jcov}  as a Lyapunov function:
 $$\begin{aligned} \frac{d}{d t} J_{\operatorname{cov}}(p, \mathcal{A}) &
   =\sum_{i=1}^n \frac{\partial}{\partial p_i} J_{\operatorname{cov}}(p,
 \mathcal{A}) \dot{p}_i \\ & =\sum_{i=1}^n 2
 m_i\left(p_i-c_{\mathcal{A}_i}\right)^T
 \left(-k_{p,i}\left(p_i-c_{\mathcal{A}_i}\right)\right) \\ & =-2 k_{p,i}
 \sum_{i=1}^n m_i\left\|p_i-c_{\mathcal{A}_i}\right\|^2, \end{aligned}$$ where
 $m_i=\int_{\mathcal{A }_i} \varphi_i(q) d q$, we have convergence towards the
 centroid by the LaSalle's invariance principle. By introducing a time-varying
 $\varphi(q)$ function, we need to do some considerations. A cyclic back and
 forth motion can be introduced by~\eqref{eq:rho}, however, by construction it
 can happen only in the proximity of the goal (it depends on $d_2$ value) and if
 the goal positions of two or more robots are close one each others.

 Live-lock can be also introduced by~\eqref{eq:wp}. In this case, since we do
 not have deadlock conditions, by construction it can occur only if two or more
 robots have a close goal position, that is, $\|e_i-e_j\|< \Delta_{ij}$, which
 violates our assumption. In fact the other possible case would be when the
 robot is constantly  hindered by other robots to reach its goal position.
 However, this is not possible since deadlock does not occur, as we have shown,
 and because each robot can be deflected from its goal position (see
 \textit{case a}). Notice that the role of the reset condition $\bar{p}_i = e_i$
 in~\eqref{eq:wp} plays a crucial role since in practice slow dynamics in $p_i$
 changes may lead to periodic motions. 

 Since all the robots converge asymptotically to their own centroid
 $c_{\mathcal{A}_i}$ by the LaSalle's principle, and deadlock/live-lock (if
 $\|p_i - e_i\|> r_{s,i}$) do not occur, then all the robots will converge to
 their centroids' positions (or, if live-lock occurs, inside
 $\mathcal{B}(e_i,r_{s,i})$), which satisfy $\|c_{\mathcal{A}_i}-e_i\|<r_{s,i}$,
 hence the proof is complete.

\end{proof}

\begin{remark}[Numerical Simulations and Asynchrony]
    The simulations in Sec.~\ref{sec:Simulation results} are implemented in
    discrete time and with discretized cells $\mathcal{A}_i$, since computing
    the centroid position $c_{\mathcal{A}_i}$ in closed form is not always
    possible. Notice that the computation of the centroid position
    $c_{\mathcal{A}_i}$ can be done in a distributed fashion. In fact each
    $i$--th robot, by relying on its position and encumbrance, and positions and
    encumbrances of the neighbors, can compute its cell $\mathcal{A}_i$. Once
    the cell geometry is defined each robot discretizes its own cell and assigns
    a weight to each point $q \in \mathcal{A}_i$, by means of
    $\varphi_i(q)$~\eqref{eq:phi}. By approximating the integrals to compute the
    centroid positions with finite sums, each robot computes its centroid
    position $c_{\mathcal{A}_i}$.  Because of this approximation,
    $\beta_{i,\min}$ has to be set together with the cell discretization d$x$.
    While $k_{p,i}$ has to be selected on the basis of the time discretization
    d$t$. In particular, to preserve safety and at the same time to allow the
    robots to generate the input asynchronously, we can verify that it is
    sufficient to have $k_{p,i}$d$t \leq
    \frac{\|p_i-c_{\mathcal{A}_i}\|-\delta_i}{\|p_i - c_{\mathcal{A}_i}\|}$, by
    rewriting equation~\eqref{eq:lb} in discrete form. It means that by
    selecting $k_{p,i}$d$t = \frac{\|p_i-c_{\mathcal{A}_i}\|-\delta_i}{\|p_i -
    c_{\mathcal{A}_i}\|}$ the $i$--th robot reaches  the position $p_i(t+1) =
    p_i(t) +\frac{\|p_i(t)-c_{\mathcal{A}_i}\|
    -\delta_i}{\|p_i(t)-c_{\mathcal{A}_i}\|}(p_i-c_{\mathcal{A}_i})$ in one
    single time step, which is a safe spot by construction, even if the position
    of the neighbors are not updated for one time step. Moreover, since $\|p_i -
    c_{\mathcal{A}_i}\|<r_{s,i}$, we can conclude that $k_{p,i}$d$t \leq
    1-\frac{\delta_i}{r_{s,i}}$. By considering  d$t$ equal to the computational
    time, this is not a strict assumption according to the computational time
    required by the algorithm (see Sec.~\ref{sec:subcomput}). In the simulations
    in Sec.~\ref{sec:Simulation results} we selected d$x = 0.075$~(m), d$t =
    0.033$~(s). 
\end{remark}

\subsection{Learning policy with Lloyd-based support (LLB)}
\label{sec:learning}

The choices behind the selection of $\beta_i$~\eqref{eq:rho} and $\bar{p}_i$~\eqref{eq:wp} are crucial to obtain acceptable performance and convergence. In the previous section we introduced a way to achieve convergence with acceptable performance. However, the selection of $\beta_i$~\eqref{eq:rho} and $\bar{p}_i$~\eqref{eq:wp} is far from optimal, i.e., a lower travel time can be achieved. Hence we propose to select these values by means of a learned function. By relying on learning-based techniques the convergence cannot be guaranteed anymore, but the safety is still ensured, both during training and testing, as it does not depend on $\beta_i$ and $\bar{p}_i$. While prior works on learned navigation tasks require to properly shape the reward function to avoid collisions and at the same time to converge to the goal,
our method allows to focus on learning how to reach the goal as quickly as possible. 

Similarly to~\cite{serragomez2023a}, we learn a high-level attention-based navigation policy that encodes the goal relative position and the relative positions and velocities of the other robots in the neighborhood (notice that the velocity is computed by relying on the past and actual positions). The learned latent representation is mapped to the parameters of a diagonal Gaussian distribution that we sample to obtain $\beta_i$ and $\bar{p}_i$.  

\subsubsection{RL Formulation}
Similarly to \cite{brito2021go}, the observation vector of each robot $i$ is composed of its own information and relative states of all robots $j$ within its sensing range:
\begin{equation}
\begin{aligned}
    &{s}^t_i = [\Delta p_{e,i}, \dot{p}_i, \delta_{i}, l_i],  \\
    &{s}^t_j = [\Delta p_{j,i}, \dot{p}_j, \delta_{j}], \quad  \forall j \in \mathcal{N}_i,
\end{aligned}
\end{equation}
where ${s}^t_i$ and ${s}^t_j$ are the available information that robot $i$ has of itself and other robots $j$ in its sensing range.
Notice that these information are obtained by relying on the onboard sensors.  
The relative positions of the goal and other robots are represented by $\Delta p_{e,i}$ and $\Delta p_{j,i}$ 
While the hyper-parameter $l_i$ 
is initialized as $l_i = 0$, it becomes equal to $1$ when $p_{i} \in \mathcal{B}(e_i,\varepsilon)$ and it stays equal to $1$ until $p_{i} \notin \mathcal{B}(e_i,\varepsilon + d)$, where $d>0$ is an arbitrary distance. All positions and velocities are represented in polar coordinates. We denote the joint observed state of all robots in $\mathcal{N}_i$ by $\mathcal{O}^t_i$.


We seek to learn the optimal policy for each robot $i$: $\pi_i(s^t_i,
\mathcal{O}^t_i)$, that maps its observation of the environment to an action
vector: ${a}^t_i = (\Delta\bar{p}_i,\beta_i)$ with $\Delta\bar{p}_i =
\bar{p}_i-p_i$, eventually guiding the robot towards its goal. The action space
is bounded by $\Delta\bar{p}_i \in \mathcal{B}(0,r_{s,i})$, while we constrained
$\beta_i \in [0.1, 0.5]$.


We design the reward function to motivate the ego-robot to reach and maintain its goal position while allowing other robots to reach theirs.
The reward attributed to each robot $W_i$ is defined as:

\begin{equation}
W_i({s}_i^{t+1},\mathcal{O}_i^{t+1},{a}_i^t,{s}_i^{t},\mathcal{O}_i^{t}) = W_{t,i} + W_{g,i},
    \end{equation}
    where  $ W_{t,i} = w_t - \Delta p_{e,i}$, 
and

\begin{equation*}\
 W_{g,i} = \begin{cases}
   
        w_g & \quad \text{if } l_{i}^{t+1}-l_{i}^t=1 \\
        0 & \quad \text{if } l_{i}^{t+1}-l_{i}^t=0 \\
        -w_g & \quad \text{if } l_{i}^{t+1}-l_{i}^t=-1
\end{cases}
\end{equation*}
with $\Delta p_{e,i}$ being the increment of distance from robot $i$ to its
goal, and $w_g \in \mathbb{R}^{+}$and $w_t\in\mathbb{R}^{-}$being the reward for
robot $i$ achieving its goal, and the punishment for each time step that all
robots have not achieved their goal, respectively. Note that there is no need
for punishing collisions. Since the low-level controller always tracks the
centroid of the cell, which now depends on the learned parameters $\bar{p}_i =
p_i + \Delta\bar{p}_i$ and $\beta_i$, collisions cannot happen by definition.
The episode ends when all the robots are within a distance $\epsilon$ of their
goals or the time step limit has been reached. 

\subsubsection{Policy Network Architecture}

Similar to \cite{serragomez2023a} we learn a high-level attention-based policy that outputs the action vector ${a}'_i = (\Delta\bar{p}'_i,\beta'_i)$ conditioned to its observation of the environment.
The architecture of the learned policy, shown in Figure \ref{fig:architecture} follows the encoder-decoder network paradigm~\cite{serragomez2023a,lee2019set}. The available information of the environment (e.g., the goal and a variable number of agents inside the sensing range) is encoded through a self-attention block (SAB) that learns how each element's representation should be affected by the presence of others. Then a multi-headed attention block is used to pool the resulting latent representations conditioned to $s^t_i$. The resulting vector is then mapped through a fully connected layer FC  to the parameters of a diagonal Gaussian distribution and an estimate of the state-value function. The diagonal Gaussian distribution is then sampled to obtain ${a}'_i = (\Delta\bar{p}'_i,\beta'_i)$, which is remapped to the action-space bounds. The state-value estimate is used only during training by the algorithm used to train the architecture. 
    \begin{figure}[t]
    \centering
     \includegraphics[width=1\columnwidth]{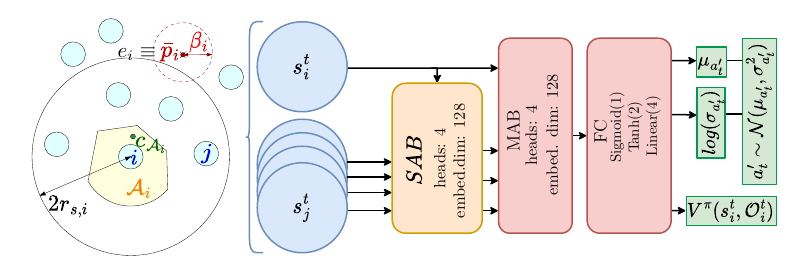}
    \caption{Overview of the learned policy network. The observed robots are encoded through a self-attention block  (SAB) and decoded with a multi-headed attention block (MAB) and three fully connected layers (FC)  
    that take $s_i$ as query. The resulting vector is mapped to the parameters of the diagonal Gaussian distribution, $\mathcal{N}(\mu_{a'_t},\sigma^2_{a'_t})$, and value function estimate, $V^{\pi}(s^t_i,\mathcal{O}^t_i)$.%
    }
    \label{fig:architecture}
\end{figure}
We employ the Proximal Policy Optimization (PPO) for training our neural network~\cite{schulmanPPO} with parameter-sharing. In our implementation~\cite{liang2018rllib}, we utilize a combination of the surrogate loss and the KL-divergence term to ensure training stability. Additionally, an entropy regularization term is integrated to foster exploration, as outlined in \cite{Haarnoja17}. For a deeper dive into the algorithm's details, we direct readers to~\cite{schulmanPPO}.





\subsubsection{Further considerations}
To guide and stabilize learned policy, we introduce a number of modifications both during training and testing. Firstly, to 
smooth our trajectories, we estimate $a_i^t$ by running an unbiased exponential recency-weighted average \cite{sutton2018reinforcement}:
\begin{equation}
    \begin{aligned}
    &a_i^{t+1} = 
    \frac{\alpha}{o^t}a'_i 
    + (1-\frac{\alpha}{o^t})a^t_i \\
    &o^{t+1} = o^t + \alpha(1-o^t)\text{, with $o^0=0$,}
    \end{aligned}
\end{equation}
where $\alpha$ is a hyperparameter that modulates how reactive is the policy to the output of the learned network. 
Secondly, we also assume that pushing effect (see Theorem~\ref{pr:conv}) is only necessary when the robot has not completed its tasks. Therefore, the output $\beta'_i$ is clipped to the lower values $[0.1,0.28]$ if $l_i = 1$, and it is clipped to the higher values $[0.32,0.5]$ otherwise. 
These values are included in the range $\beta_{\min}$ and $\beta^D$. A small gap between the lower values and the higher values is introduced in order to be sure to induce the pushing effect. 
This allows the robot to be more aggressive when the goal has not been achieved, and let other robots pass while it is close to its goal. We also substitute $\bar{p}'_i$ by the goal position $e_i$ when $l_i = 1$ to ensure stationary behavior when robots are at their goal, making the pushing effect easier to learn.
We refer to the overall method as Learning Lloyd-based method (LLB).
In the simulation results in Sec.~\ref{sec:Simulation results}, we show the advantages of LLB with respect to the pure learning method without the Lloyd-based support.




\subsection{Dynamic constraints}

Until now we only discussed the case of holonomic robots without constraints, however it is possible to implement the proposed algorithm also on unicycle, car-like robots, omnidirectional robots and aerial robots. 

In fact, as we discussed, the holonomic robot has just to pursue a point position (i.e., the centroid $c_{\mathcal{A}_i}$). Similarly to~\cite{zhu2022decentralized}, we use an MPC formulation to generate control inputs for the robots, which are compliant with the robot dynamics and that at each iteration generates a trajectory that remains inside the cell $\mathcal{A}_i$. In this way we can still guarantee safety at every time step. 
Each robot $i$ solves the following MPC problem (we omit the $i$ indexing to simplify the notation),
\begin{subequations}
\label{eq: MPC problem_formulation}
\begin{align}\label{eq:cost_generic}
\min_{{x}_k,{u}_k}& \sum^{N_t}_{k=0} J_k({x}_k,{u}_k, c_{\mathcal{A}}),\\
\label{eq:MPC dyn_constraints_generic}\textrm{s.t.: } & {x}_{k} = f({x}_{k-1},{u}_{k-1}) ,\\
\label{eq:MPC bounds_generic}&\,{x}_k\in \mathcal X, \, {u}_k\in \mathcal U,\\
 \label{eq: MPC voronoi cell constraint}
 & p_k \in \mathcal{A},\\
 \label{eq:init_generic}
 &{x}_0 = {x}_{\textrm{init}},
\end{align}
\end{subequations}
where $k$ indicates the time index, $N_t$ is the number of time instants in the MPC problem, $J$ is the cost function, ${x}_{k}$ and ${u}_{k}$ are the state and inputs respectively, $f({x}_{k-1},{u}_{k-1})$ represents the robot dynamical constraints,  ${x}_k\in \mathcal X, \, {u}_k\in \mathcal U$ are the state and input constraints respectively, and ${x}_0 = {x}_{\textrm{init}}$ is the initial condition. The cost $J$ can be chosen arbitrary as far as it incentives the tracking of the desired velocity. For each time instant $k$ a possible formulation can be the following:
\begin{equation}\label{eq:Jmpc}\small
    J_k = \left(v_k-v^D\frac{c_{\mathcal{A}} - p_k}{\|c_{\mathcal{A}} - p_k\|}\right)^\top \left(v_k-v^D\frac{c_{\mathcal{A}} - p_k}{\|c_{\mathcal{A}} - p_k\|}\right)  + {u}_k^{\top} Q {u}_k,
\end{equation}
where ${v}_k$ is the robot's velocity in the $x$-$y$ plane, $v^D$ is the target absolute velocity that may be defined as a function of the distance from the centroid, i.e., $v^D = v^D(\|c_{\mathcal{A}} - p_k\|)$ and $Q$ is a positive definite matrix of weights.

It is important to note that since we have now introduced potentially complex robot dynamics~\eqref{eq:MPC dyn_constraints_generic} as well as actuator and state constraints~\eqref{eq:MPC bounds_generic}, we do not guarantee the recursive feasibility of the MPC problem~\eqref{eq: MPC problem_formulation}. In fact, the problem can be unfeasible if for example, the sensing range $r_{s,i}$ is too small and the robots have quite limited braking capability (in this case the constraints \eqref{eq: MPC voronoi cell constraint} can be violated, compromising the safety of the system). However, in practice, for adequate sensing range and target velocity values, this is not an issue as is shown in Sec.~\ref{sec:Experimental results}.

\section{Simulation results}
\label{sec:Simulation results}

The proposed approach is tested in simulation in several scenarios. We also provide a comparison with state of the art approaches. Through the numerical simulations we aim to further verify the scalability, the robustness, and the flexibility of the proposed approach. The multimedia material that accompanies the paper provides some videos of the simulation results.
The source code can be found at \texttt{https://github.com/manuelboldrer/RBL}.

The parameters are selected as follows: $r_{s,i} = 1.5$~(m). Because of
numerical issues due to the fact that the centroid position is not computed
analytically but numerically, we also set a minimum value for $\beta_{i,\min} =
0.1$. Notice that the values for the parameters $d_1=d_3, d_2 = d_4$ have to
depend on $r_{s,i}$ and $\beta_i^D$ for a correct convergence (see
Theorem~\ref{pr:conv}). For the simulations we used $d_1 = d_3 = 0.1$ and $d_2 =
d_4 =  d_{p_i,c_{\mathcal{S}_i}}-d_1$ unless specified
otherwise. Notice that, even if the chosen parameters may not always satisfy the
sufficient conditions for convergence (depending on $\delta_i$), we always
obtained convergence from the simulation  results. It indicates that the
sufficient conditions provided for convergence are conservative, at least for
scenarios that are not overly crowded (e.g., with $\eta < 0.1$). \looseness=-1

The performance metrics that we considered include the maximum travel time (max time), which measures how much it takes to all the robots to converge towards their goal positions, the average speed of each robot during the mission (for each robot the mission stops when it reaches its goal position), and the robots success rate (RSR), computed as the ratio of robots that safely reached the goal to the total number of robots. 
The simulation results section is organized as follows: 1. we provide simulation results of RBL algorithm in crossing circle scenario, half crossing circle scenario and random room scenario. 2. We compared the proposed LLB approach with a pure learning algorithm without the Lloyd support and also with RBL. The comparison was done in multiple scenarios by repeating each experiment $20$ times to have a statistical evidence. 3. We provide simulation results of RBL algorithm by testing different dynamical constraints, i.e., unicycle and car-like. 4. To test the robustness and reliability of RBL we added randomness in the selection of robot encumbrance $\delta_i$, $\beta^D$ and $k_{p,i}$ and we run $100$ simulations by accounting also different number of robots in the scene. 5. We provide a comparison with the state of the art. In particular we compared RBL by considering the unicycle case with LB~\cite{boldrer2020lloyd}  and RL-RVO~\cite{han2022reinforcement}, which performs better than SARL~\cite{chen2019crowd}, GA3C-CADRL~\cite{everett2018motion} and NH-ORCA~\cite{alonso2013optimal}. Then we compared RBL in the holonomic case with RVO2~\cite{van2011rvo2}, SBC~\cite{wang2017safety}, GCBF+~\cite{zhang2024gcbf+},
BVC~\cite{zhou2017fast}, DMPC~\cite{luis2020online} and LCS~\cite{park2022online}. 5. Finally, we provided some insights on the computational cost of the proposed algorithm.
\subsection{Crossing circle (RBL)}
We report several simulations in the crossing circle scenario. In Fig.~\ref{fig:crosscircle} we depicted the trajectories obtained by considering $N = [5,10,25,50,300]$ homogeneous holonomic robots. The radius of the circle is $R_c = [10,10,10,10,15]$~(m). The value of $\beta_i^D = 0.5$, $\forall i = 1,\dots, N$. The robot encumbrance $\delta_i = [0.35,0.35,0.35,0.35,0.1]$~(m), and  $k_{p,i} = 6$, $\forall i = 1,\dots, N$.

 \begin{figure*}[t]
  \centering
  \setlength{\tabcolsep}{0.05em}
  \begin{tabular}{ccccc}
    \includegraphics[width=.4\columnwidth]{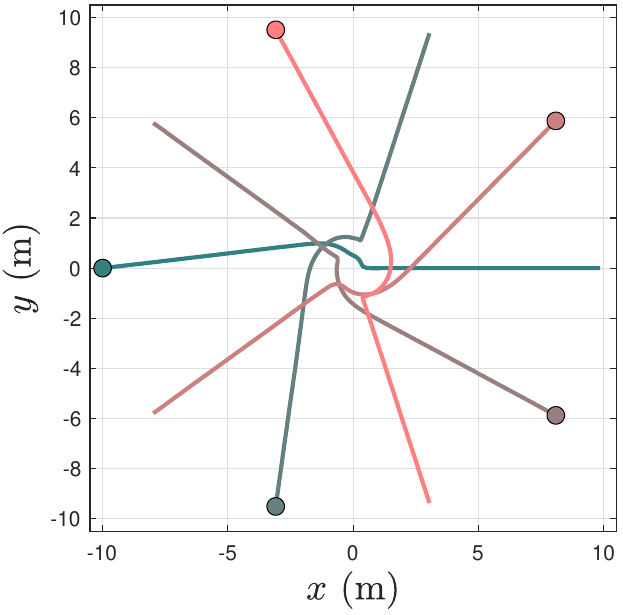} &
  \includegraphics[width=.4\columnwidth]{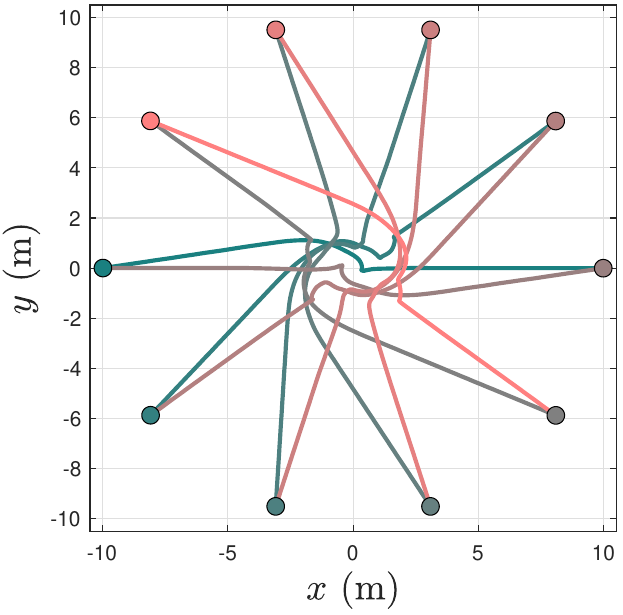}   &
   \includegraphics[width=.4\columnwidth]{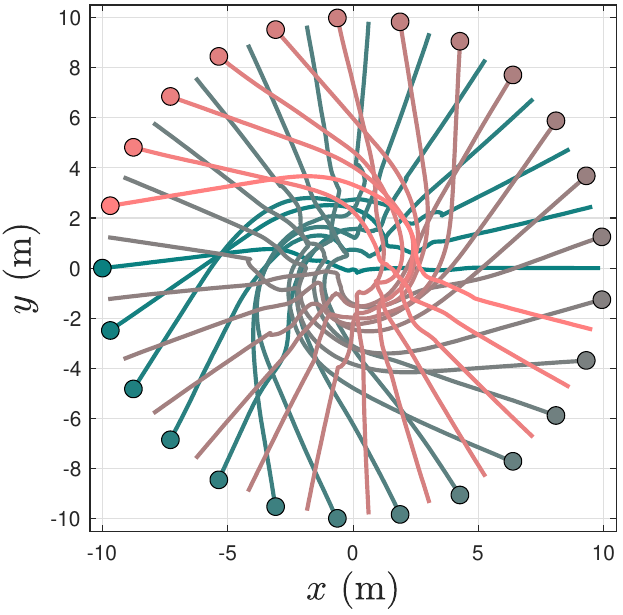} &
  \includegraphics[width=.4\columnwidth]{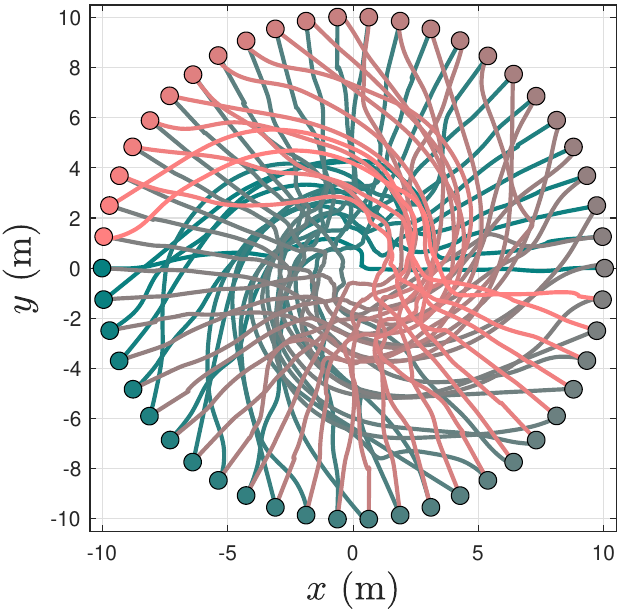}  
  \includegraphics[width=.4\columnwidth]{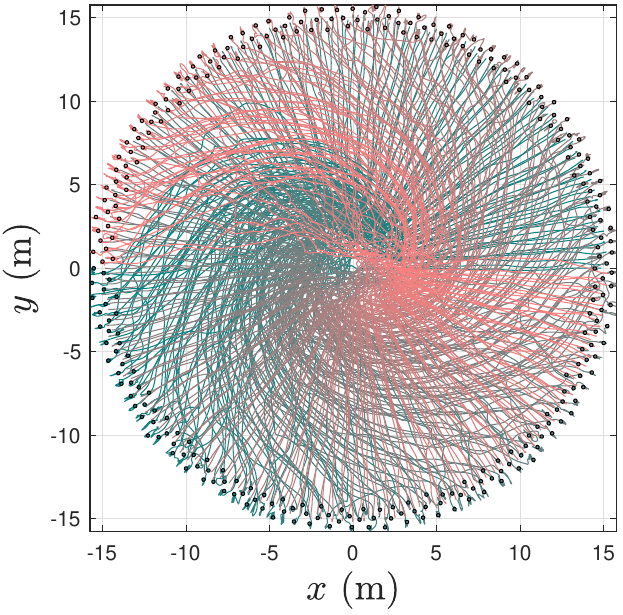}
    \end{tabular}
    \caption{Circle crossing: simulation results with $N= [5,10,25,50,300]$ holonomic robots. The value $\beta^D=0.5$ is the same for all the robots. }
  \label{fig:crosscircle}
\end{figure*}

\begin{table}[t]
    \centering
    \renewcommand{\tabcolsep}{0.1cm} 
    \begin{tabular}{|c|c|c|c|c|} \hline
        N   &  $\eta$ & max time (s) & mean speed (m/s) & succ. rate \\ \hline
        5   & 0.0061 & 5.18 & 3.96 & 1.00\\ \hline
        10  & 0.0122  & 5.91 & 3.73 &1.00 \\ \hline
        25  & 0.0306   & 7.98 & 2.91 &1.00 \\ \hline
        50  & 0.0612  & 11.09 & 2.40 &1.00 \\ \hline
        300 &  0.0133  & 30.76 & 1.52 &1.00 \\ \hline            
    \end{tabular}
     \vspace{10pt}
    \caption{Circle crossing: quantitative data form the simulation in Figure~\ref{fig:crosscircle}.}
    \label{tab:cc}
\end{table}

In Table~\ref{tab:cc} we report the quantitative results, where $N$ is the number of robots and $\eta = \frac{\text{Area}_\text{robots}} {\text{Area}_\text{tot}}$ is what we called the \emph{crowdness factor}, which is the ratio between the sum of the area occupied by the robots 
and the total mission space area, it is an indicator of the challenging and intricate nature of the considered scenario.

In Figure~\ref{fig:simul1} we show other two simulations, where we picked a
random encumbrance radius for each robot $\delta_i \in [0.1,0.5]$. On the left
hand side we selected $\beta^D_i =0.5$, $\forall i = 1,\dots,N$, while on the
simulation on the right hand side we selected $\beta^D_i \in [0.2,1.5]$ randomly
for each robot. Figure~\ref{fig:rho} shows the effects of the choice of
$\beta^D$ on the travel time for each robot, as we expected, by increasing
$\beta^D$ the travel time increases as well.

\begin{figure}[t]
  \centering
    \setlength{\tabcolsep}{0.05em}
  \begin{tabular}{cc}
    \includegraphics[width=.5\columnwidth]{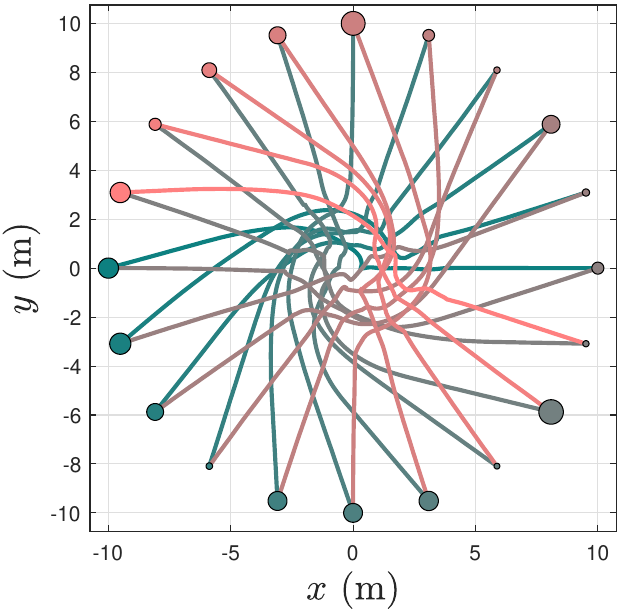} &
  \includegraphics[width=.5\columnwidth]{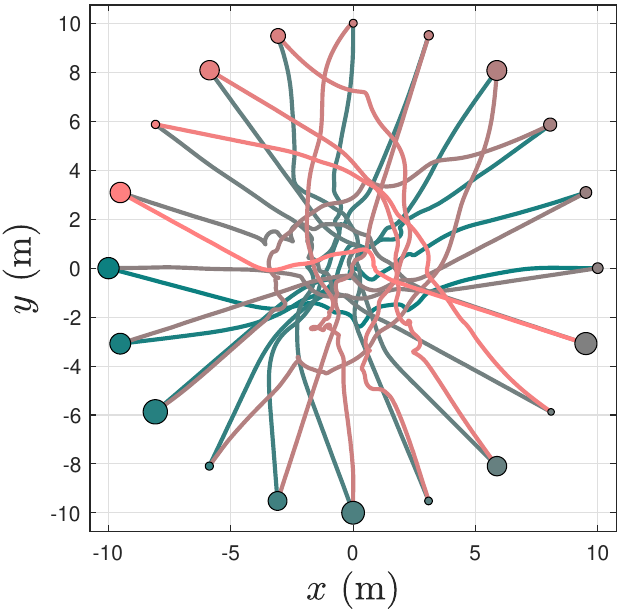}   \\
  \end{tabular}
    \caption{Circle crossing: simulation results with $N=20$ holonomic robots with random encumbrance $r$. On the left hand side $\beta^D_i = 0.5$ for all the robots, On the right hand side $\beta^D_i$ is random and different from robot to robot.  }
  \label{fig:simul1}
\end{figure}

\begin{figure}[t]
    \centering
    \includegraphics[width=1\columnwidth]{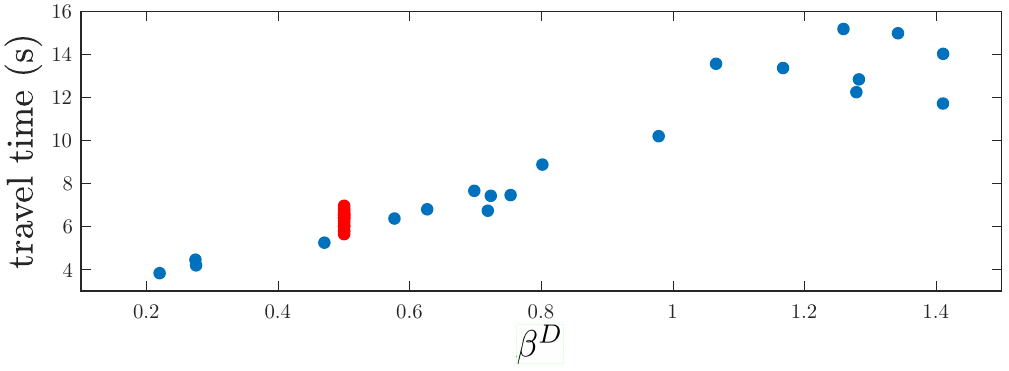}
    \caption{Dependence of the travel time with respect to $\beta^D_i$. Each dot represents the performance of one robot. In red the left hand side simulation in Figure~\ref{fig:simul1}, in blue the right hand side simulation in Figure~\ref{fig:simul1}. }
    \label{fig:rho}
\end{figure}

\subsection{Half crossing circle (RBL)}
  Similarly to the crossing circle case, we report some simulations for the half crossing circle scenario. In this case, the goal position for each robot is shifted by an angle of $\pi + \gamma$, rather than being exactly out of phase by $\pi$. In Figure~\ref{fig:halfcrosscircle}, we show that even if the robots follow the right hand side rule according to~\eqref{eq:wp}, our algorithm can also manage this challenging situations, where the emergent behavior is a clockwise rotation. In Table~\ref{tab:hcc} we report the quantitative results. 
\begin{figure*}[t]
  \centering
  \setlength{\tabcolsep}{0.05em}
  \begin{tabular}{ccccc}
    \includegraphics[width=.4\columnwidth]{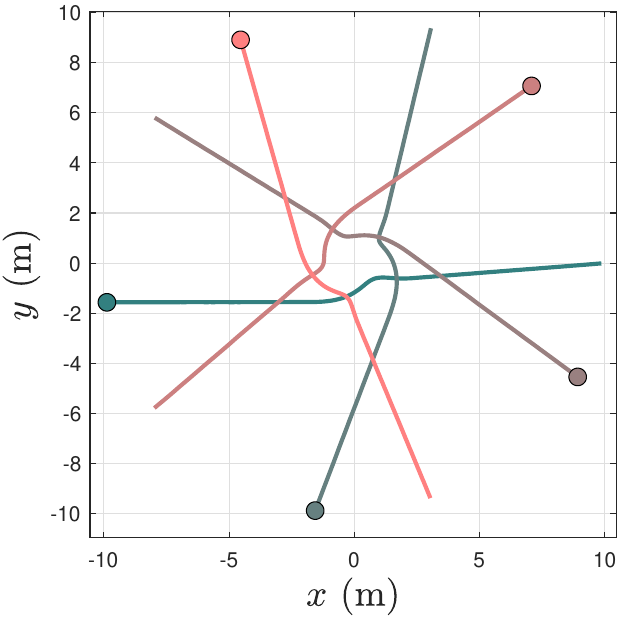}   &
    \includegraphics[width=.4\columnwidth]{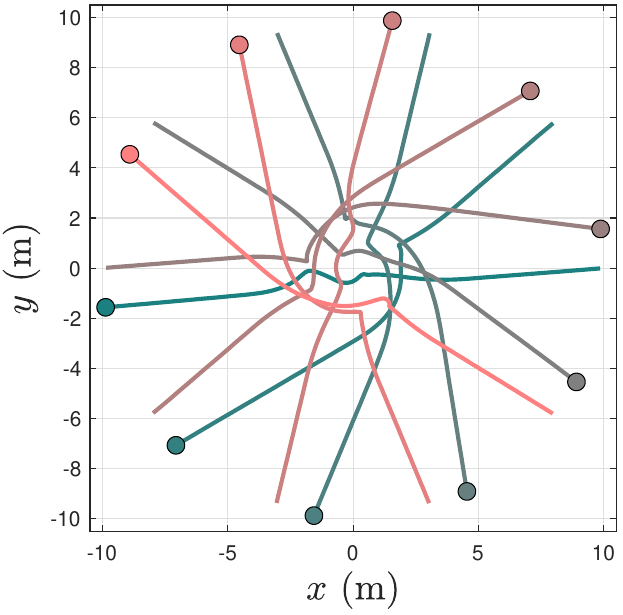}  &
    \includegraphics[width=.4\columnwidth]{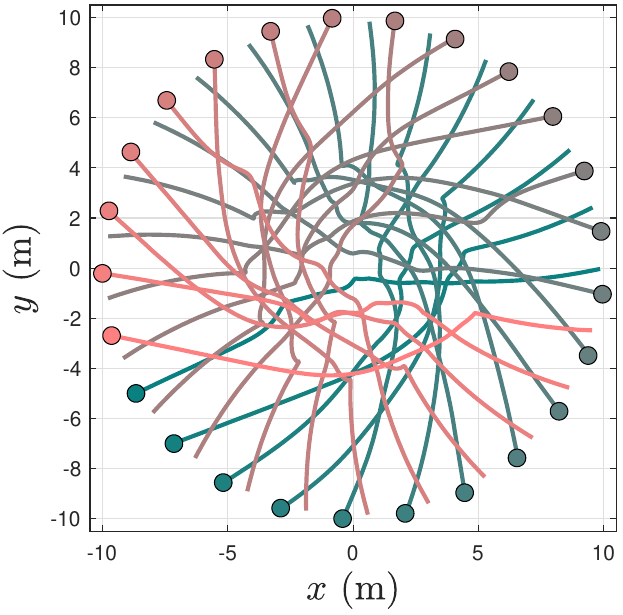}  &
    \includegraphics[width=.4\columnwidth]{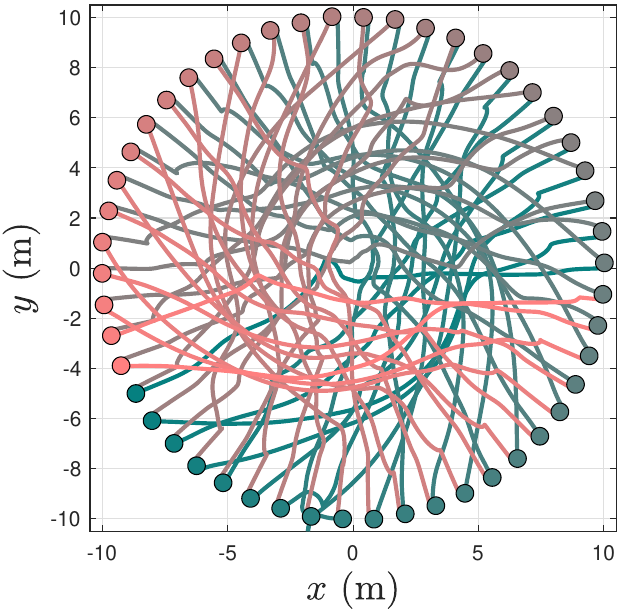}  &
    \includegraphics[width=.4\columnwidth]{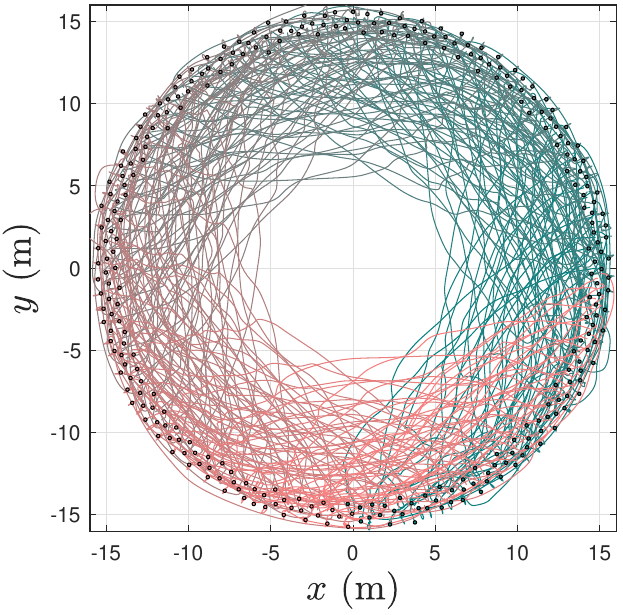} 
    \end{tabular}
    \caption{Half circle crossing: simulation results with $N= [5,10,25,50,300]$, holonomic robots. The value $\beta^D=0.5$ is the same for all the robots. The goal position is rotated with respect to the crossing scenario of a value $\gamma =[\pi/20,\pi/20,\pi/6,\pi/6, \pi/2]$. }
  \label{fig:halfcrosscircle}
\end{figure*}

\begin{table}[t]
\renewcommand{\tabcolsep}{0.1cm} 
    \centering
    \begin{tabular}{|c|c|c|c|c|} \hline
        N   &  $\eta$ & max time (s) & mean speed (m/s) & succ. rate \\ \hline
        5  & 0.0061  & 5.05&  3.95& 1.00\\ \hline
        10  & 0.0122  & 5.44&  3.77&1.00\\ \hline
        25  & 0.0306   & 6.47&  3.43&1.00\\ \hline
        50  & 0.0612  & 7.01&  2.76&1.00\\ \hline
        300 &  0.0133  & 16.59& 1.91 &1.00\\ \hline            
    \end{tabular}
     \vspace{10pt}
    \caption{Half circle crossing: quantitative data form the simulation in Fig.~\ref{fig:halfcrosscircle}.}
    \label{tab:hcc}
\end{table}

\subsection{Random room (RBL)}
In this scenario we report two simulations. The initial location and the final goal location are selected randomly. In Figure~\ref{fig:random}, on the left hand side,  $\eta_1 = 0.157$ and $N=20$ homogeneous robots, on the right hand side $\eta_2 = 0.135$ and $N=300$ heterogeneous (random size) robots. In both cases the robots successfully converge to a neighborhood of the goal positions.
\begin{figure}[t]
  \centering
  \renewcommand{\tabcolsep}{0.01cm} 
  \begin{tabular}{cc}
    \includegraphics[width=.49\columnwidth]{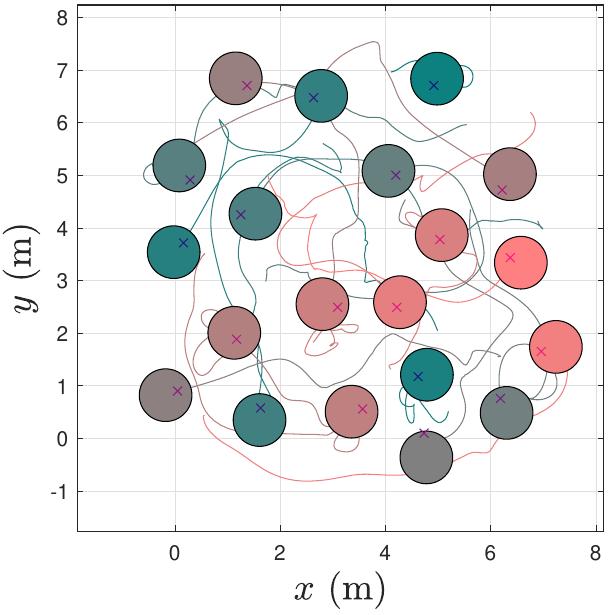} &
  \includegraphics[width=.5\columnwidth]{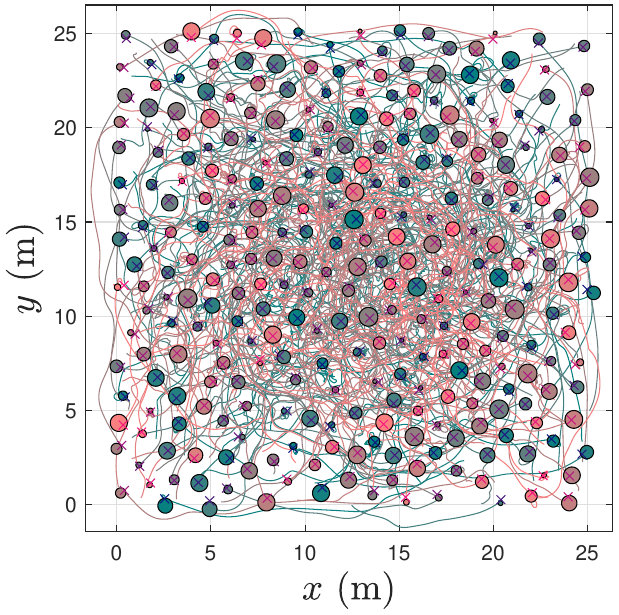} 
  \end{tabular}
    \caption{(left) Room 7x7 (m$^2$): simulation results with $N=20$ holonomic robots with crowdness factor $\eta = 0.157$. (right) Room 25x25 (m$^2$): simulation results with $N=300$ holonomic robots with crowdness fator $\eta = 0.135 $ and random encumbrance $\delta$.}
  \label{fig:random}
\end{figure}

\subsection{Learning approach (LLB)}

We implemented the Learning Lloyd-based (LLB) algorithm introduced in Section~\ref{sec:learning}, to verify the effectiveness of using a Lloyd-based support on learning-based algorithms. In Table~\ref{tab:learn} we compared the same learning algorithm, with the Lloyd support (LLB), without the support of the LB algorithm (we indicated it in Table~\ref{tab:learn} as Pure Learning algorithm), and also with the RBL algorithm.


We made few changes to train the learning approach without Lloyd support. In fact in this case, we only learn a point relative to the robot $\Delta\tilde{p}_i \in \mathcal{B}(0,r_{s,i})$, to follow with velocity  $\dot{p}_i = v_{i,\text{max}}\frac{\Delta\tilde{p}_i}{r_{s,i}}$. Since collisions under this policy are possible, additional reward signal $r_{\text{coll}}$ is also added to the reward function to punish collision events. All other conditions have been maintained. 
We observed that the addition of LB during training simplifies the problem. Pure learning method requires trading off collision avoidance with fast goal convergence. Instead, adding the safety LB layer during training allows to only consider fast convergence towards the goal. 
We run multiple simulations in different scenarios. It results that the LLB was able to achieve a success rate of $1.00$, while pure learning presents failures even in simple scenarios (e.g., with $N=5$). With respect to RBL, the LLB performs better in symmetric scenarios (i.e., circle scenarios) and slightly  worse in the random room scenarios. In the case of symmetric scenarios this is due to the fact that RBL needs some time to break perfect symmetries by engaging the right hand rules (the transient time depends on the parameters settings), while LLB is able to break symmetries immediately as can be seen from Figure~\ref{fig:circleRL}.  In the figure we depict the trajectories obtained from LLB in the crossing circle scenarios reported also in Table~\ref{tab:learn}.

\begin{figure}
  \setlength{\tabcolsep}{0.05em}
    \centering
     \begin{tabular}{ccc}
    \includegraphics[width=.33\columnwidth]{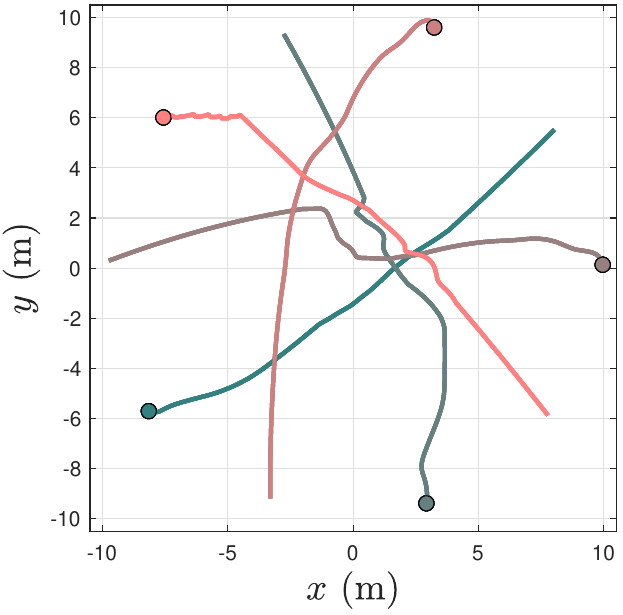} &
    \includegraphics[width=.33\columnwidth]{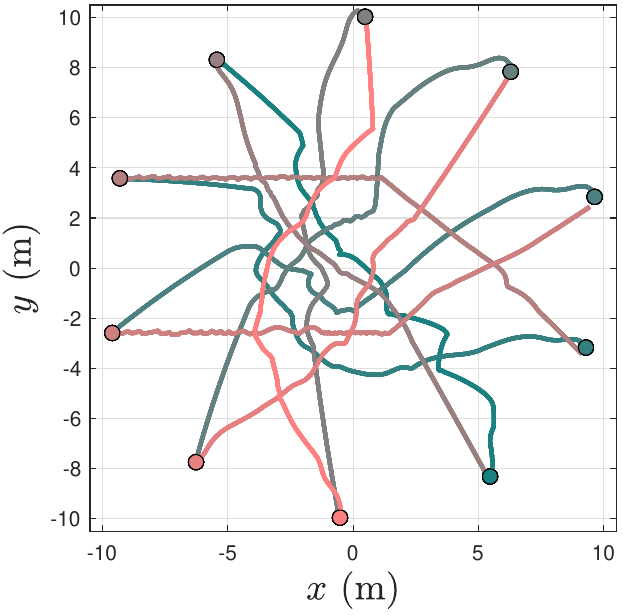}&
    \includegraphics[width=.33\columnwidth]{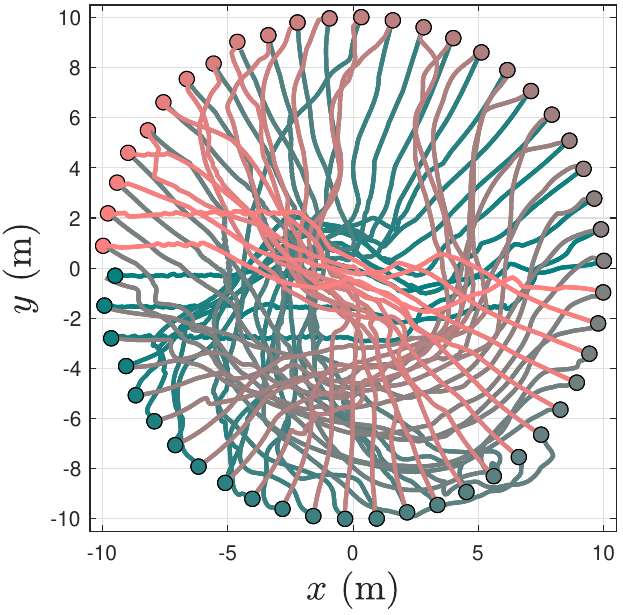} \\
    \end{tabular}
    \caption{Crossing circle: simulation results with $N= [5,10,50]$ holonomic robots, by using LLB approach. }
    \label{fig:circleRL}
\end{figure}

 \begin{table}
 \renewcommand{\tabcolsep}{0.1cm} 
    \centering
    \begin{tabular}{|c|c|c|c|c|}
    \hline 

RBL & $N$ &  $\eta$ & max time (s) & succ. rate \\ \hline
   circle ($10$ m) & 5 & 0.0045  & 6.56  & 1.00 \\ \hline
   room ($9\times 9$ m$^2$)& 5 &0.0289& 2.25  & 1.00 \\ \hline
   circle ($10$ m) & 10 & 0.0090   &  7.78 & 1.00 \\ \hline
    room ($9\times 9$ m$^2$)& 10 & 0.0349& 2.74 & 1.00 \\ \hline
    circle ($10$ m) & 50 &  0.0450 & 17.65 & 1.00 \\ \hline
    room ($15\times 15$ m$^2$) & 50 & 0.0628  & 12.36 & 1.00 \\ \hline \hline
    Pure  Learning & $N$ &  $\eta$ & max time (s) & succ. rate \\ \hline
   circle ($10$ m) & 5 & 0.0045  & 5.25  & 0.94  \\ \hline
   room ($9\times 9$ m$^2$)& 5 &0.0289& 2.79  & 0.56 \\ \hline
   circle ($10$ m) & 10 & 0.0090   &  6.38 & 0.77 \\ \hline
    room ($9\times 9$ m$^2$)& 10 & 0.0349& 3.31& 0.04 \\ \hline
    circle ($10$ m) & 50 &  0.0450 & - & 0.00 \\ \hline
    room ($15\times 15$ m$^2$) & 50 & 0.0628  & - & 0.00 \\ \hline \hline
     LLB & $N$ & $\eta$ & max time (s) & succ. rate \\ \hline
   circle ($10$ m) & 5 & 0.0045  &  5.47   & 1.00\\ \hline
   room ($9\times 9$ m$^2$) & 5 &0.0289  &  2.42  &1.00\\ \hline
   circle ($10$ m) & 10 & 0.0090 & 6.43 &1.00\\ \hline
    room ($9\times 9$ m$^2$)& 10 & 0.0349 &  3.37  &1.00\\ \hline
    circle ($10$ m) & 50 & 0.0450 &  13.35  &1.00\\ \hline
    room ($15\times 15$ m$^2$) & 50 & 0.0628  &  13.67  &1.00\\ \hline

    \end{tabular}
    \vspace{10pt}
    \caption{Simulation results for rule-based Lloyd (RBL), Pure Learning and Learning with Lloyd support (LLB). We considered $N=[5,10,50]$ holonomic robots in the crossing circle scenario and in the random room scenario, with robot encumbrance $\delta_i = 0.3$~(m). The maximum time is an average of the successful simulations. 
    }
    \label{tab:learn}
\end{table}  


\subsection{Dynamical constraints (RBL)}
We also tested our algorithm with different dynamical models, i.e., the unicycle and the car-like. In Figure~\ref{fig:circleNH} we depicted the results in the crossing circle scenario for $N  = [5,10,25]$, for the unicycle model (top) and for the car-like model (bottom). In both cases we set a maximum and desired forward velocity $v^D = 1.5$~(m/s). The quantitative results are reported in Tables~\ref{tab:unicc} and~\ref{tab:carcc}. 

\begin{figure}
  \setlength{\tabcolsep}{0.05em}
    \centering
     \begin{tabular}{ccc}
    \includegraphics[width=.33\columnwidth]{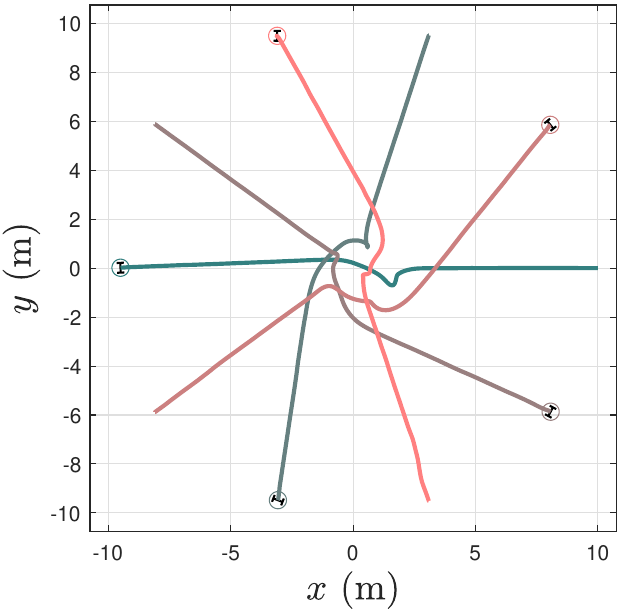} &
    \includegraphics[width=.33\columnwidth]{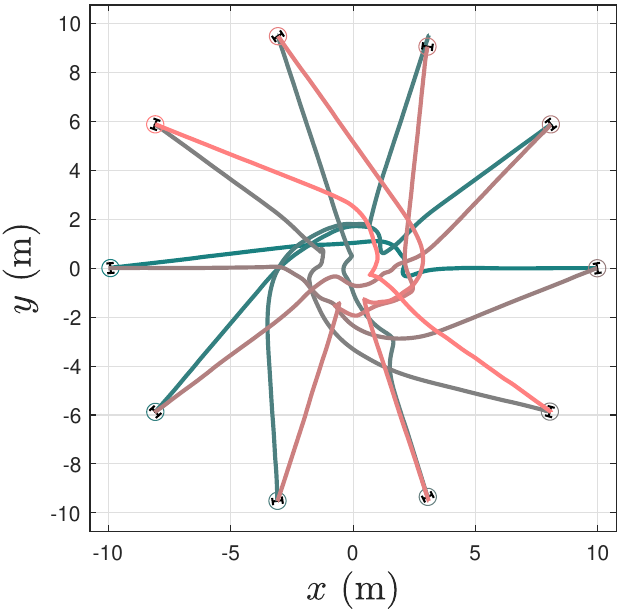}&
    \includegraphics[width=.33\columnwidth]{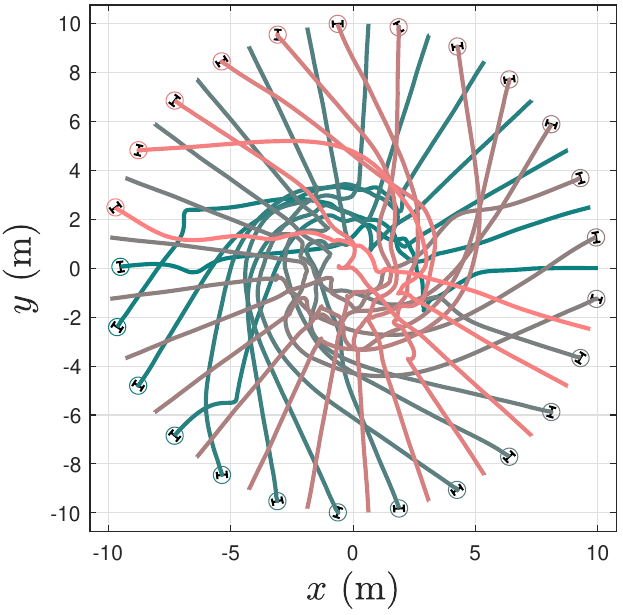} \\
             \includegraphics[width=.33\columnwidth]{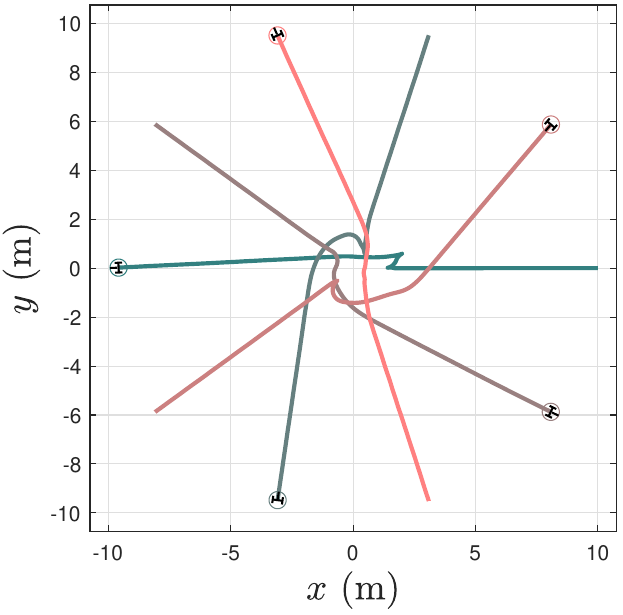} &
    \includegraphics[width=.33\columnwidth]{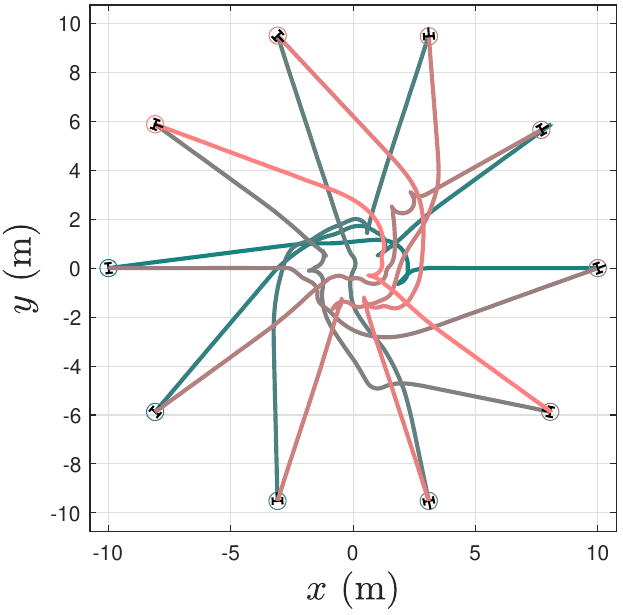}&
    \includegraphics[width=.33\columnwidth]{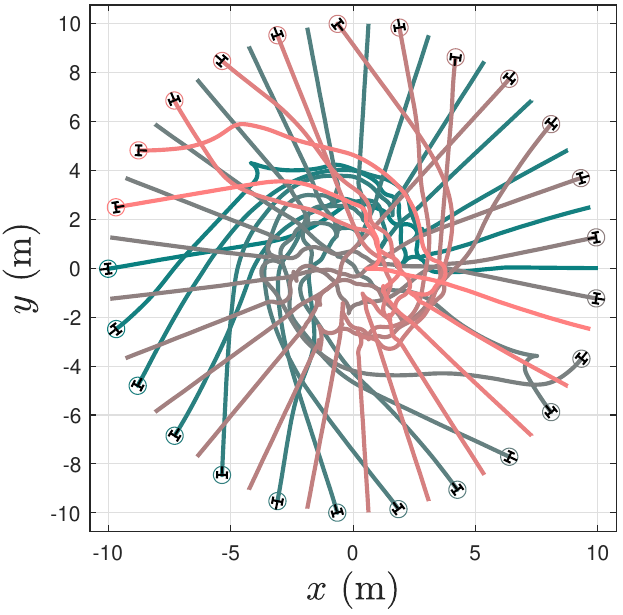} 
    \end{tabular}
    \caption{Crossing circle: simulation results with $N= [5,10,25]$ non-holonomic robots. At the top the unicycle model, at the bottom the car-like model. }
    \label{fig:circleNH}
\end{figure}

\begin{table}[t]
    \centering
    \renewcommand{\tabcolsep}{0.1cm} 
    \begin{tabular}{|c|c|c|c|c|} \hline
        N   &  $\eta$ & max time (s) & mean speed (m/s) & succ. rate \\ \hline
        5  & 0.006125& 21.50 & 1.09  & 1.00\\ \hline
        10  & 0.01225& 26.60 & 1.01 & 1.00 \\ \hline
        25  & 0.0306& 32.90 & 0.82 & 1.00 \\ \hline           
    \end{tabular}
     \vspace{10pt}
    \caption{Circle crossing unicycles: quantitative data from the simulation in Fig.~\ref{fig:circleNH} (top).}
    \label{tab:unicc}
\end{table}

\begin{table}[t]
    \centering
    \renewcommand{\tabcolsep}{0.1cm} 
    \begin{tabular}{|c|c|c|c|c|} \hline
        N   &  $\eta$ & max time (s) & mean speed (m/s) & succ. rate \\ \hline
        5   & 0.006125 & 26.90 & 0.86  & 1.00\\ \hline
        10  & 0.01225 & 32.00 & 0.80 &1.00 \\ \hline
        25  & 0.0306 & 44.40 & 0.64 &1.00 \\ \hline     
    \end{tabular}
     \vspace{10pt}
    \caption{Circle crossing car-like: quantitative data from the simulation in Fig.~\ref{fig:circleNH} (bottom). }
    \label{tab:carcc}
\end{table}



\subsection{Robustness and reliability (RBL)}

For the crossing circle and random room in the holonomic case we run additional simulations to test the robustness of the algorithm to the change in parameters. According to the data reported in~\cite{fan2020distributed}~Sec.~$5.4$, it is clear that in learning-based solutions performance, safety and convergence degrade by changing parameters such as the dimension of the robots and the desired velocity.
In the following we show an analysis of robustness by considering a heterogeneous scenario and by selecting randomly the encumbrances of the robots $\delta_i \in [0.1,0.5]$, the value of $\beta^D_i \in [0.2,0.75]$ and $k_{p,i} \in [3,6]$. We run $100$ simulations with $N = [20,40,100]$ robots. The simulation results are reported in Table~\ref{tab:statistics}, the success rate is always equal to $1.00$.

\begin{table}
    \centering
    \footnotesize
    \renewcommand{\tabcolsep}{0.05cm} 
    \begin{tabular}{|c|c|c|c|c|c|}
    \hline 
  RBL  & $N$ & mean of $\eta$ & max time (s) & mean speed (m/s)& RSR \\ \hline
   circle ($4$ m) & 20 & 0.112 & 8.87$\pm$ 0.98& 1.70 $\pm$ 0.10& 1.00 \\ \hline
   room ($7\times 7$ m$^2$)& 20 & 0.115 & 5.33 $\pm$ 1.13  & 1.79 $\pm$ 0.25 &1.00 \\ \hline
   circle ($7$ m) & 40 & 0.073 & 12.96 $\pm$ 0.84 & 1.80 $\pm$ 0.11 & 1.00 \\ \hline
    room ($9\times 9$ m$^2$)& 40 & 0.139 & 10.24 $\pm$ 2.16 & 1.46 $\pm$ 0.07&1.00 \\ \hline
    circle ($16$ m) & 100 & 0.0352 & 24.57 $\pm$ 1.15& 2.03 $\pm$ 0.10&1.00 \\ \hline
    room ($15\times 15$ m$^2$) & 100 & 0.125 &21.83 $\pm$ 4.89& 1.46 $\pm$ 0.07 &1.00 \\ \hline 
  
    \end{tabular}
    \vspace{10pt}
    \caption{Simulation results for $N=[20,40,100]$ robots in the crossing circle scenario and in the random room scenario, with $\delta_i \in [0.1,0.5]$~(m), $\beta^D_i \in [0.2,0.75]$ and $k_{p,i} \in [3,6]$. }
    \label{tab:statistics}
\end{table}

\subsection{Comparison with state of the art}
As we already mentioned, our main advantages  with respect to the algorithms proposed in the literature reside in the combination of four aspects:

1. We can achieve success rate of $1.00$ in very crowded scenarios. 
2. Each robot needs to know only its position $p_i$, its encumbrance $\delta_i$, its goal position $e_i$, its neighbours' positions $p_j$ and encumbrance $\delta_j$ (if $\| p_i - p_j \| \leq  2 r_{s,i}$, then $j \in \mathcal{N}_i$);
3. The control inputs for each robot can be generated in an asynchronous fashion without affecting safety and convergence.
4. We can account for heterogeneous conditions, i.e., different robots' encumbrances and different robots' dynamics and velocities; 

The combination of these features are not provided by any other algorithm in the literature. In the following, we provide a comparison with some state of the art algorithms.

Firstly, we considered the non-holonomic (unicycle) case, we compared the performance of our rule-based Lloyd algorithm (RBL) with LB~\cite{boldrer2020lloyd} and RL-RVO~\cite{han2022reinforcement}, which demonstrates superior performance and success rate compared to other algorithms such as SARL~\cite{chen2019crowd}, GA3C-CADRL~\cite{everett2018motion}, NH-ORCA~\cite{alonso2013optimal}.

We considered the crossing circle scenario $R_c =4$~(m) with $N =
[6,20,30,100]$, $\delta_i = 0.2$~(m) ($\delta_i = 0.1$~(m) for $N=100$), a
maximum forward velocity $v_{i,\max} =1.5$~(m/s), and limits on the
accelerations of $a_{i,\max} =1$~(m/s$^2$). As a result (see
Table~\ref{tab:sota}), RBL executes better in terms of success rate with respect
to the other algorithms (in this case the success rate is the ratio between the
number of missions where the robots converged safely and the total number of
missions simulated). On the other hand, when it succeeds, RL-RVO has on average
better performance. This is also due to the fact that RL-RVO relies on stronger
assumptions (each robot has to collect unobservable information).

We also propose a comparison with other \textit{reactive} algorithms for the
holonomic robots case, such as RVO2~\cite{van2011rvo2},
SBC~\cite{wang2017safety}, GCBF+~\cite{zhang2024gcbf+} and also with recent
\textit{predictive planning} methods such as
BVC~\cite{zhou2017fast,abdullhak2021deadlock}, DMPC~\cite{luis2020online} and
LCS~\cite{park2022online}. Even if some of these algorithms may have better
performance in terms of time to reach the goal in certain scenarios, they lack
of guarantees of live-lock and deadlock avoidance, hence, even in simple
scenarios they may fail the mission. In~\cite{park2022online} (Sec. IV, Fig. 5)
it is clear that the success rate of the mission decreases considerably for
$N>50$, as well for BVC, DMPC and LCS. Also RVO2~\cite{van2011rvo2} presents
deadlock issues in symmetric conditions (crossing circle) and in crowded random
room scenarios with high crowdness factor, e.g., we tested a random room
scenario with $\eta = 0.452$, it results in a SR~$= 0.65$. For the SBC
approach~\cite{wang2017safety}, in~Sec. VII, it is clearly stated that the
algorithm suffers of what they called type III deadlock and live-lock as well.
We tested it in a scenario with $\eta = 0.452$, it results in SR~$= 0.00$. We
obtain the same SR~$= 0.00$ for more advanced variants such as
GCBF+~\cite{zhang2024gcbf+}. While in the same scenario our RBL reached SR $=
1.00$. Finally, for SBC we also provided a comparison of the performance in
simple crossing circle scenarios with $N = [5,10,20,50]$ (see
Table~\ref{tab:sbccomparison}). Since SBC performance may vary on the basis of
the parameter tuning, we considered the provided scenario, where robots
encumbrance $\delta_i = 0.0675$ (m), and maximum velocity $v_{i,\max}= 0.2$
(m/s) and we compared it with our method in the same conditions. Contrary to
other methods seen before, this can be considered a fair comparison, since SBC
does not use unobservable information as our method. In this case, RBL performs
better also in terms of time to reach the goals.

\begin{table}
\renewcommand{\tabcolsep}{0.05cm} 
\small
    \centering
    \begin{tabular}{|c|c|c|c|c|c|}
    \hline 
     & $N$ &$\eta$  &  max time (s) & mean speed (m/s) & succ. rate  \\ \hline
     &  5 & 0.00253& 42.70 & 0.131 & 1.00 \\ 
   SBC &  10 & 0.00506 &45.96 &0.125 &1.00 \\ 
     &  20& 0.0101& 56.33 & 0.118 &1.00 \\ 
     & 50&  0.0252 & 71.08 & 0.106 & 1.00 \\ \hline
       & $N$ & $\eta$ &  max time (s) & mean speed (m/s)& succ. rate  \\ \hline
       &  5 &  0.00253 & 29.79 & 0.195 &1.00  \\ 
   RBL (ours)  &  10 &0.00506& 32.17& 0.192 &1.00  \\ 
    &  20& 0.0101 & 39.46 & 0.185& 1.00 \\ 
  & 50& 0.0252 & 49.36 & 0.180 & 1.00 \\ \hline
   
    \end{tabular}
    \vspace{10pt}
    \caption{Circle crossing: quantitative data with unicycle model. We compared our method (RBL) with SBC~\cite{wang2017safety}. }
   \label{tab:sbccomparison}
\end{table}  



In conclusion, with respect to the state of the art, our RBL algorithm often
achieves similar performance in terms of time to reach the goal, nevertheless it
consistently attains a success rate of 1.00, and it does so without relying on
unobservable information (e.g., future intentions of the other robots), the neighbours
velocities, or synchronization between the robots. For the best of our
knowledge, this is not provided by any other distributed algorithm in the
literature. In this section
we provided a relevant amount of comparisons with state
of the art algorithms (including also the ones that uses observation).
The main result that we obtained can be synthesized as follows: our algorithm
outperforms all the tested algorithms in terms of success rate, i.e., we
always obtained SR $= 1.00$ in all the tested scenarios, while
LB, RL-RVO, SARL, GA3C-CADRL, NH-ORCA, RVO2,
SBC, BVC, GCBF+, DMPC, and LCS obtained a SR $< 1.00$.

We want to underline that if a reliable communication network and a good computational power are available, then IMPC-DR~\cite{chen2022recursive} and ASAP~\cite{chen2023asynchronous} are valid solutions that provide a success rate of SR~$= 1.00$ similarly as our.
\begin{table}
\renewcommand{\tabcolsep}{0.03cm} 
\small
    \centering
    \begin{tabular}{|c|c|c|c|c|c|}
    \hline 
     & $N$ &$\eta$  &  max time (s) & mean speed (m/s) & succ. rate  \\ \hline
     &  6 & 0.015& 7.10 & 1.19 & 1.00 \\ 
   RL-RVO &  20 & 0.05 &13.84 &0.73 &0.71 \\ 
     &  30& 0.075& 21.89 & 0.62 & 0.10 \\ 
     & 100&  0.0625 & - & - & 0.00 \\ \hline
       & $N$ & $\eta$ &  max time (s) & mean speed (m/s)& succ. rate  \\ \hline
       &  6 &0.015 & 13.42 & 0.66 &1.00  \\ 
  
   LB   &  20 &0.05 & -& - &0.00  \\ 
    &  30& 0.075 & - & -& 0.00 \\ 
  & 100& 0.0625 & - & - & 0.00 \\ \hline
     & $N$  & $\eta$ & max time (s) & mean speed (m/s)  & succ. rate \\ \hline
     &  6 &0.015 & 9.86 & 0.89 &1.00 \\ 
RBL (ours)  &  20&0.05 & 17.85 & 0.65 & 1.00 \\ 
    & 30&0.075 &26.16  &0.59 & 1.00 \\ 
         & 100&  0.0625 & 28.76 & 0.39 &1.00 \\ \hline

    \end{tabular}
    \vspace{10pt}
    \caption{Circle crossing: quantitative data with unicycle model. We compared
    our method (RBL) with LB~\cite{boldrer2020lloyd} and
    RL-RVO~\cite{han2022reinforcement} method, since SARL~\cite{chen2019crowd},
    GA3C-CADRL~\cite{everett2018motion} and NH-ORCA~\cite{alonso2013optimal}
    underperform it according to~\cite{han2022reinforcement} (Table II). Notice
    that in RL-RVO, each robot has to collect unobservable information, i.e., the method relies on a communication network. }
   \label{tab:sota}
\end{table}  

\subsection{Computational time}\label{sec:subcomput}
The computational time of the proposed approach depends mainly on $r_{s,i}$ (assuming a fixed space discretization step d$x = 0.075$~(m)), since we compute the centroid position numerically. We tested the algorithm in Python on an AMD Ryzen 7 5800H. By considering that $r_{s,i}$ should be limited to a maximum amount of $2$~(m) (notice that $r_{s,i}$ is half the sensing radius capability) a real-time implementation is largely feasible. In fact we measured for $r_{s,i} = 2.0$~(m) an average computational time of $t_c = 10$~(ms), for $r_{s,i} = 1.5$~(m) we measured $t_c = 5.5$~(ms), while for $r_{s,i} = 1.0$~(m) the computational time reduces to $t_c = 2$~(ms).
The addition of the MPC to account for dynamic constraints has an impact on the computational time. However, despite a clear increase in the computational time, it does not prevent to be used for real-time applications (as it is shown in the experimental result Sec.~\ref{sec:Experimental results}). We measured an increase of $16$~(ms) on average.

\section{Experimental results}
\label{sec:Experimental results}
We tested our RBL algorithm on different robotic platforms. In particular, car-like robots, unicyle robots, omnidirectional robots, and aerial robots. Notice that in the following implementations we relied on a communication network to share the neighbors' positions and encumbrances. Nevertheless, by means of onboard sensors such as lidars, cameras or uvdar~\cite{walter2019uvdar}, the neighbors' positions can be easily inferred (Assumption~\ref{ass:assumption1}), hence the functioning of the algorithm would be comparable (or better since less delays are introduced) in the absence of a communication network.  In the car-like case, we used $N=3$ scaled-down car-like robots with $\delta_i = 0.2$~(m), the robots consist of a modified Waveshare JetRacer Pro AI kit. We selected $r_{s,i} = 1.5$~(m), $\beta^D = 0.5$, $d_1=d_3= 0.15$ and, $d_2=d_4=0.45$. 
 In the following, we report the results obtained from a crossing circle scenario and a random room scenario with $N = 3$. In both the scenarios the robots converge safely to their goal positions, in accordance with the theory and the simulation results.  The dynamic constraint~\eqref{eq:MPC dyn_constraints_generic} is now described as the following kinematic bicycle model:
\begin{align*}
\begin{bmatrix}\dot{x}\\\dot{y}\\\dot{\eta}\\\dot{v}\end{bmatrix}&=\begin{bmatrix}v\cos{\eta}\\v\sin{\eta}\\\frac{v \tan(\delta)}{l}\\h(\tau, v)\end{bmatrix},
\end{align*}
where $x,y,\eta,v$ are respectively the $x$ and $y$ position of the rear axle, the orientation of the robot and the longitudinal velocity, $l$ is the length of the robot, $\tau$ and $\delta$ are the throttle and steering angle inputs, while $h(\tau,v)$ is the motor characteristic curve. Figs.~\ref{fig:exp1} and~\ref{fig:exp2} report two experiments in the crossing circle and random room scenarios, respectively. We depicted with solid lines the trajectories followed by the robots and with crosses the goal locations. The robots safely converged to their goal locations. In Table~\ref{tab:exp} we report the quantitative data.

\begin{figure}
\centering
\includegraphics[width=1\columnwidth]{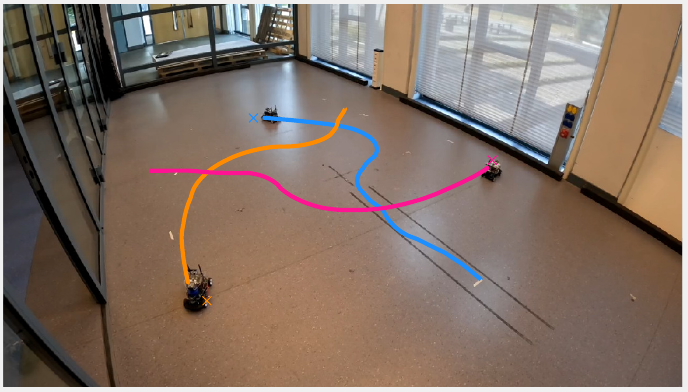}
    \caption{Experiments with $N= 3$ car-like robots in a crossing circle scenario.}
    \label{fig:exp1}
\end{figure}

\begin{figure}
\centering
\includegraphics[width=1\columnwidth]{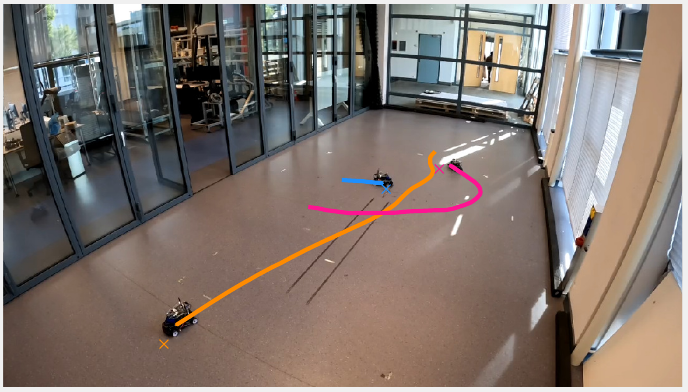}
    \caption{Experiments with $N=3$ car-like robots in a random room scenario.}
    \label{fig:exp2}
\end{figure}



\begin{table}
  \small
    \centering
    \renewcommand{\tabcolsep}{0.05cm} 
    \begin{tabular}{|c|c|c|c|c|c|}
    \hline 
  RBL  & $N$ & $\eta$ & max time (s) & mean speed (m/s)& RSR \\ \hline
   circle ($1.5$ m) & 3 & 0.0075 & 7.69 & 0.428& 1.00 \\ \hline
   room ($5\times 3$ m$^2$)& 3 & 0.014 & 9.50 & 0.448 &1.00 \\ \hline
   
    \end{tabular}
    \vspace{10pt}
    \caption{Quantitative experimental results for $N=3$ car-like robots in the crossing circle scenario and in the random room scenario with $v^D=0.5$~(m/s) (see Figure~\ref{fig:exp1} and~\ref{fig:exp2}). }
    \label{tab:exp}
\end{table}

For the unicycle robots case, we used $N=3$ Clearpath Jackals with $r_{s,i}=
1.5$, $\delta_i = 0.3$~(m), $\beta^D =0.5$, $d_1=d_3=0.15$, and $d_2=d_4=0.6$.
Also in this case, we report the results obtained from a crossing circle
scenario and a random room scenario with $N = 3$. Again, the robots converge
safely to their goal locations. In this case the dynamic
constraint~\eqref{eq:MPC dyn_constraints_generic} can be written as:
\begin{align*}
  \begin{bmatrix}\dot{x}\\\dot{y}\\\dot{\theta}\\\end{bmatrix}&=\begin{bmatrix}v\cos{\eta}\\v\sin{\eta}\\
\omega \end{bmatrix}, \end{align*} where $x,y,\theta,v,\omega$ are respectively
the $x$ and $y$ position of the rear axle, the orientation of the robot, the
longitudinal and the angular velocities. Figs.~\ref{fig:exp1j}
and~\ref{fig:exp2j} report two experiments in the crossing circle and random
room scenarios, respectively. In Table~\ref{tab:expj} we report the quantitative
data. 

\begin{figure}
\centering
\includegraphics[width=1\columnwidth]{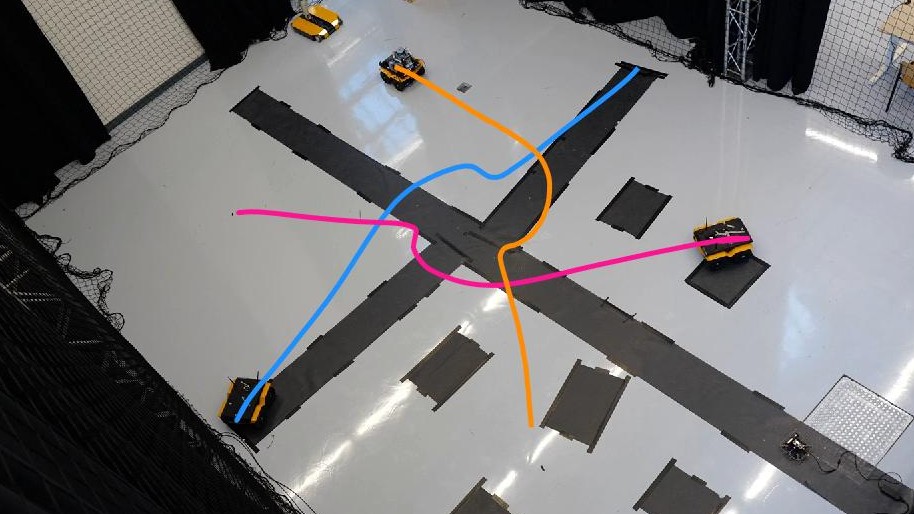}
    \caption{Experiments with $N= 3$ unicycle-like robots in a crossing circle scenario. }
    \label{fig:exp1j}
\end{figure}

\begin{figure}
\centering
\includegraphics[width=1\columnwidth]{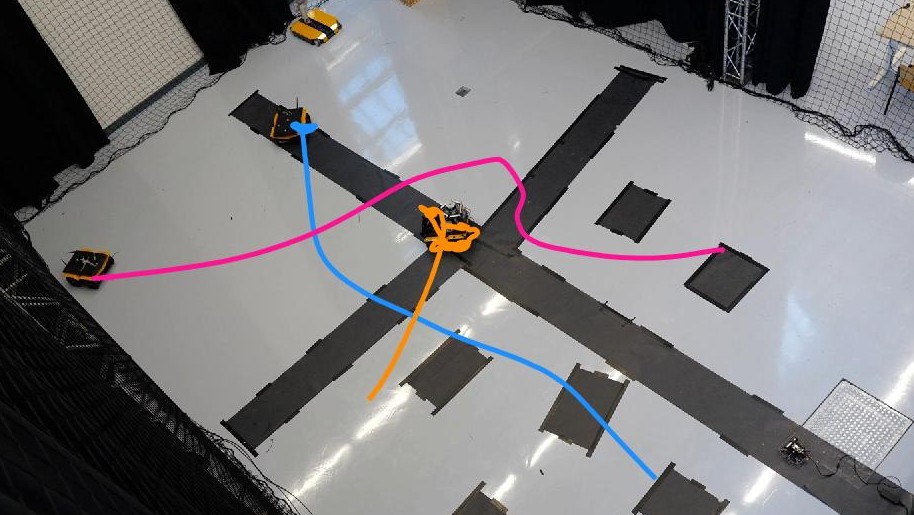}
    \caption{Experiments with $N=3$ unicycle-like robots in a random room scenario. }
    \label{fig:exp2j}
\end{figure}



\begin{table}
  \small
    \centering
    \renewcommand{\tabcolsep}{0.05cm} 
    \begin{tabular}{|c|c|c|c|c|c|}
    \hline 
  RBL  & $N$ & $\eta$ & max time (s) & mean speed (m/s)& RSR \\ \hline
   circle ($3$ m) & 3 & 0.030 & 16.90 & 0.415& 1.00 \\ \hline
   room ($7\times 4$ m$^2$)& 3 & 0.010 & 21.90& 0.442 &1.00 \\ \hline
   
    \end{tabular}
    \vspace{10pt}
    \caption{Quantitative experimental results for $N=3$ unicycle-like robots in the crossing circle scenario and in the random room scenario $v^D=0.5$~(m/s) (see Figure~\ref{fig:exp1} and~\ref{fig:exp2}). }
    \label{tab:expj}
\end{table}

We also report three experiments with heterogenous agents. In
particular, in Figure~\ref{fig:exp4_1} we considered the case
of $N=4$ robots with three Clearpath Jackals and one Clearpath
Dingo. The Clearpath Dingo is an omnidirectional robot hence
it can directly take as control inputs the velocities $\dot
{p}_i$. In Figure~\ref{fig:exp5_1}, we considered the case of
$N=5$ agents, with three Clearpath Jackals, one Clearpath
Dingo and one Human being. In this case the human being is
seen by the robots as another robot, hence the algorithm does
not change. The only difference lies in the human being
motion, which does not follow the RBL control law.
Nevertheless, the robots avoided collisions and achieved their
goal locations. Finally, in Figure~\ref{fig:exp7robots}, we
reported the obtained paths by using $N= 7$ aerial robots
(four DJI F450 and three Holybro
X500~\cite{hert2022mrs,hert2023mrs}) in the crossing circle
scenario. In this case, we considered $\delta_i = 2.50$ (m),
$r_{s,i} = 3.0$ (m), $d_1=d_3= 0.2$, $d_2=d_4=1.0$,
$\beta_i^D=1.5$ and $\eta = 0.194$. To reach the centroid
positions, we relied on the MRS UAV system~\cite{baca2021mrs}.
For the localization the robots relied on GPS. Even in the
presence of uncertainties, tracking errors, and imperfect
communication, each robot satisfied $p_k \in \mathcal{A}$ at
every time step, and the robots safely converged to their goal
regions. In Fig.~\ref{fig:distancesUAVs} we depicted the
distances as a function of time, notice that the
robots reach their goal region without exceeding the safe
minimum distance, which is set to $2\delta_i = 5~(m)$.
Videos of the experiments are provided in the multimedia material.

\begin{figure}
\centering
\includegraphics[width=1\columnwidth]{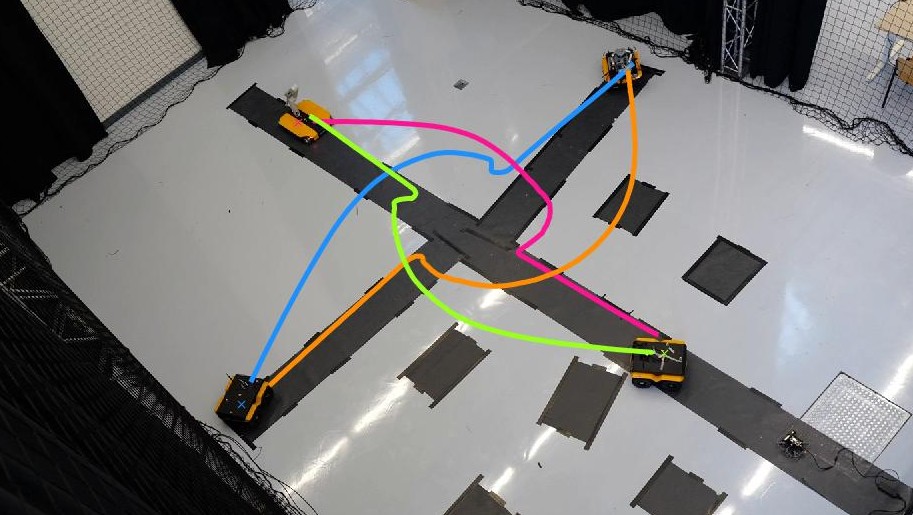}
    \caption{Experiments with $N= 4$ robots, three unicycle-like robots and one holonomic robot in a crossing circle scenario. }
    \label{fig:exp4_1}
\end{figure}

\begin{figure}
\centering
\includegraphics[width=1\columnwidth]{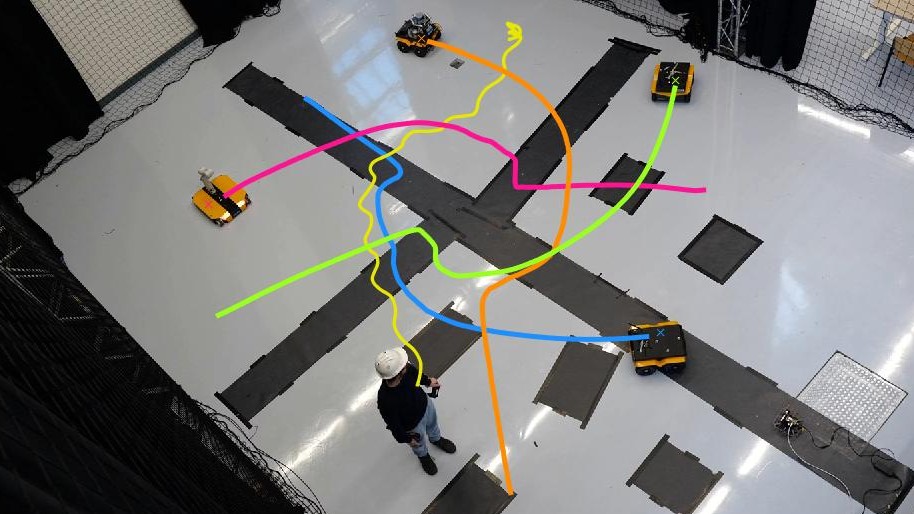}
    \caption{Experiments with $N=5$ agents, three unicycle-like robots, one holonomic robot and one human being in a crossing circle scenario. }
    \label{fig:exp5_1}
\end{figure}

\begin{figure}[t]
\centering
\includegraphics[width=1\columnwidth]{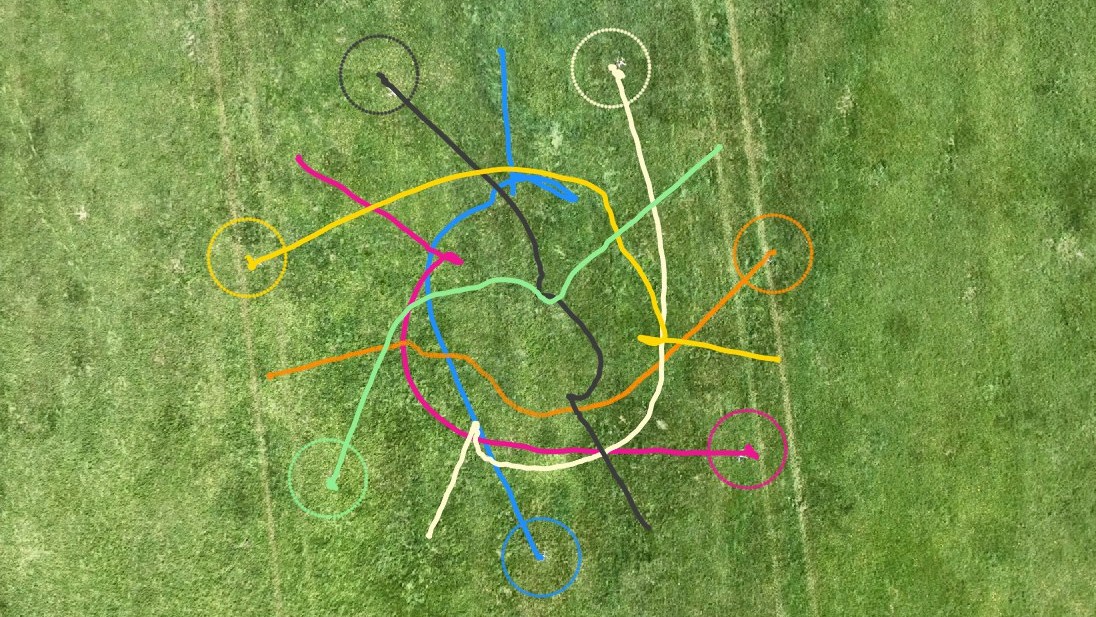}
    \caption{Experiment with $N=7$ aerial robots in a crossing circle scenario, the colored circles of radius $\delta_i = 2.5$~(m) indicate the final positions of the aerial robots. }
    \label{fig:exp7robots}
\end{figure}

\begin{figure}
\centering
\includegraphics[width=1\columnwidth]{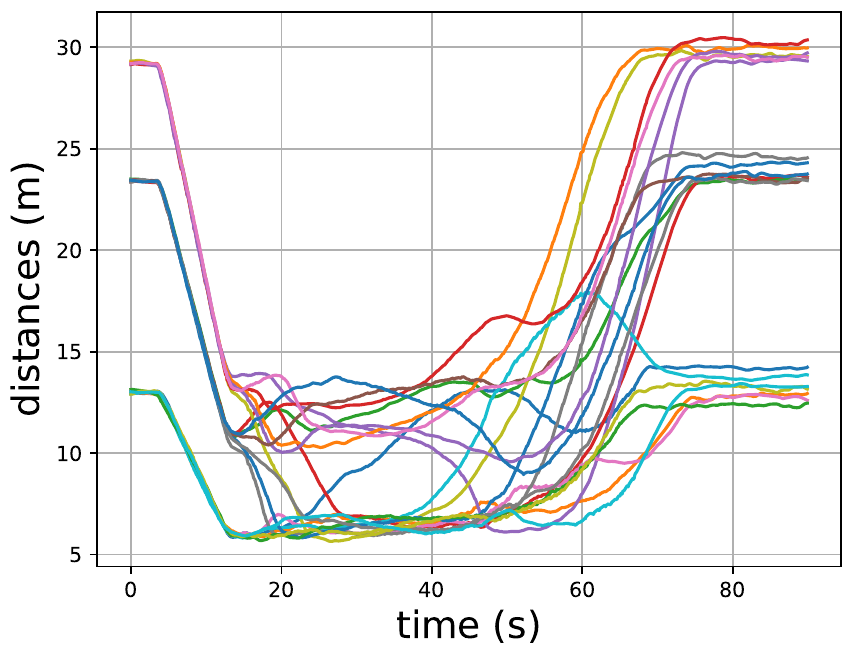}
    \caption{Experiments with $N= 7$ aerial robots, distances between robots as a function of time.}
    \label{fig:distancesUAVs}
\end{figure}

\section{Conclusions}
\label{sec:Conclusions}
This paper provided two solutions for multi-robot motion planning and control, namely rule-based Lloyd  (RBL) and Learning Lloyd-based (LLB) algorithm. Our solutions prove beneficial in a settings lacking of a centralized computational unit or a dependable communication network. In the case of RBL, we demonstrated both theoretically and through a comprehensive simulation campaign that each robot safely converges to its goal region. While through LLB, we showed that the Lloyd-based algorithm can be used as a safety-layer for learning-based algorithms. We believe that it can pave the way towards safe guaranteed learning-based algorithms and online learning as well.
We provided also experimental results for RBL approach with car-like robots,
unicycle robots, omnidirectional robots, and aerial robots on the field.

The main limitations of the proposed approach are the following: we can only
guarantee to converge in a proximity of the goal position, the parameters have
to be chosen carefully to meet the requirements to guarantee convergence, even
if we provided a rigorous way to select most of them, a tuning process is
necessary in order to impose the desired behavior to the robot motion (e.g.,
more aggressive or more conservative), the weight matrix $Q$ in the MPC cost
function~\eqref{eq:Jmpc} have to be carefully chosen in order to obtain the
desired behavior. Finally, we have to be aware that the numerical simulations
give an approximation of the centroid position. To obtain a better approximation
we have to pay in computational efficiency (by decreasing d$x$).

In the future, our objectives include addressing some of the mentioned
limitations, i.e., obtaining a closed form solution for the centroid position,
and design an adaptive law for the $d_2$ parameter in~\eqref{eq:rho} to converge
to the goal position. We plan to do experiments in scenarios that include
multiple non-cooperative or partially-cooperative agents in the mission space
(i.e., multiple human beings). We plan to do the experiments relying only on
onboard sensors to extrapolate the neighbors' positions. Finally, we plan to
further develop the LLB algorithm by testing the benefits of the LB layer during
training on other learning-based algorithms, and by synthesizing hybrid
solutions with RBL to provide guarantees of convergence.

\begin{acks}
This research is supported by the NWO-TTW Veni project HARMONIA (no. 18165).

The source code of the presented approach is available at https://github.com/manuelboldrer/RBL

The video associated with the paper can be found in the multimedia material.
\end{acks}
\bibliographystyle{SageH}
\bibliography{reference}

\end{document}